\documentclass[journal]{IEEEtran}
\usepackage[colorlinks,urlcolor=blue,linkcolor=blue,citecolor=blue]{hyperref}
\usepackage{wrapfig}
\usepackage{graphicx}
\usepackage{titlesec}
\usepackage{bbding}
\usepackage{amsthm}
\usepackage[algo2e,linesnumbered,lined,boxed,commentsnumbered]{algorithm2e}
\newtheorem{theorem}{Theorem}
\newtheorem{assumption}{Assumption}
\newtheorem{lemma}{Lemma}
 
\def\ie{\emph{i.e.\ }}

\def\mb{\mathbf}

\usepackage{array}
\usepackage{float}
\usepackage{amsthm}
\usepackage{subfig}
\usepackage{booktabs}
\usepackage{algorithm}  
\usepackage{algorithmic}
\usepackage{amssymb}
\usepackage{amsmath}
\usepackage{amssymb}
\usepackage{makecell}
  
\usepackage{bm}
\usepackage{multirow}
\usepackage{svg} 
\usepackage{amsmath,amsthm,amssymb,amsfonts}

\usepackage[capitalize]{cleveref}
\crefname{section}{Sec.}{Secs.}
\Crefname{section}{Section}{Sections}
\Crefname{table}{Table}{Tables}
\crefname{table}{Tab.}{Tabs.}
\crefname{algoritm}{Algorithm}{Algorithms}

\ifCLASSINFOpdf
\else
\fi
\begin{document}
%
\title{Inverse-Free Fast Natural Gradient Descent Method for Deep Learning}
%
%
%

\author{Xinwei~Ou, Ce~Zhu,~\IEEEmembership{Fellow,~IEEE,}
Xiaolin~Huang,~\IEEEmembership{Member,~IEEE,}
and~Yipeng~Liu,~\IEEEmembership{Member,~IEEE}
\thanks{Xinwei Ou, Ce Zhu, and Yipeng Liu (corresponding author) are with the School of Information and Communication
Engineering, University of Electronic Science and Technology of China (UESTC), Chengdu, 611731, China (e-mails: xinweiou@std.uestc.edu.cn, $ \{$eczhu, yipengliu$\}$@uestc.edu.cn).
}
\thanks{Xiaolin Huang is with the Department of Automation, Shanghai Jiao Tong University, Shanghai, 200240, China (e-mails: xiaolinhuang@sjtu.edu.cn).}}
\maketitle

\begin{abstract}
Second-order optimization techniques have the potential to achieve faster convergence rates compared to first-order methods through the incorporation of second-order derivatives or statistics. However, their utilization in deep learning is limited due to their computational inefficiency. Various approaches have been proposed to address this issue, primarily centered on minimizing the size of the matrix to be inverted. Nevertheless, the necessity of performing the inverse operation iteratively persists. In this work, we present a fast natural gradient descent (FNGD) method that only requires inversion during the first epoch. Specifically, it is revealed that natural gradient descent (NGD) is essentially a weighted sum of per-sample gradients. Our novel approach further proposes to share these weighted coefficients across epochs without affecting empirical performance. Consequently, FNGD exhibits similarities to the average sum in first-order methods, leading to the computational complexity of FNGD being comparable to that of first-order methods. Extensive experiments on image classification and machine translation tasks demonstrate the efficiency of the proposed FNGD. For training ResNet-18 on CIFAR-100, FNGD can achieve a speedup of 2.07$\times$ compared with KFAC. For training Transformer on Multi30K, FNGD outperforms AdamW by 24 BLEU score while requiring almost the same training time.

\end{abstract}


%
\IEEEpeerreviewmaketitle

\section{Introduction}
\label{sec:intro}
\begin{figure*}[h]
    \centering
    \includegraphics[scale=0.5]{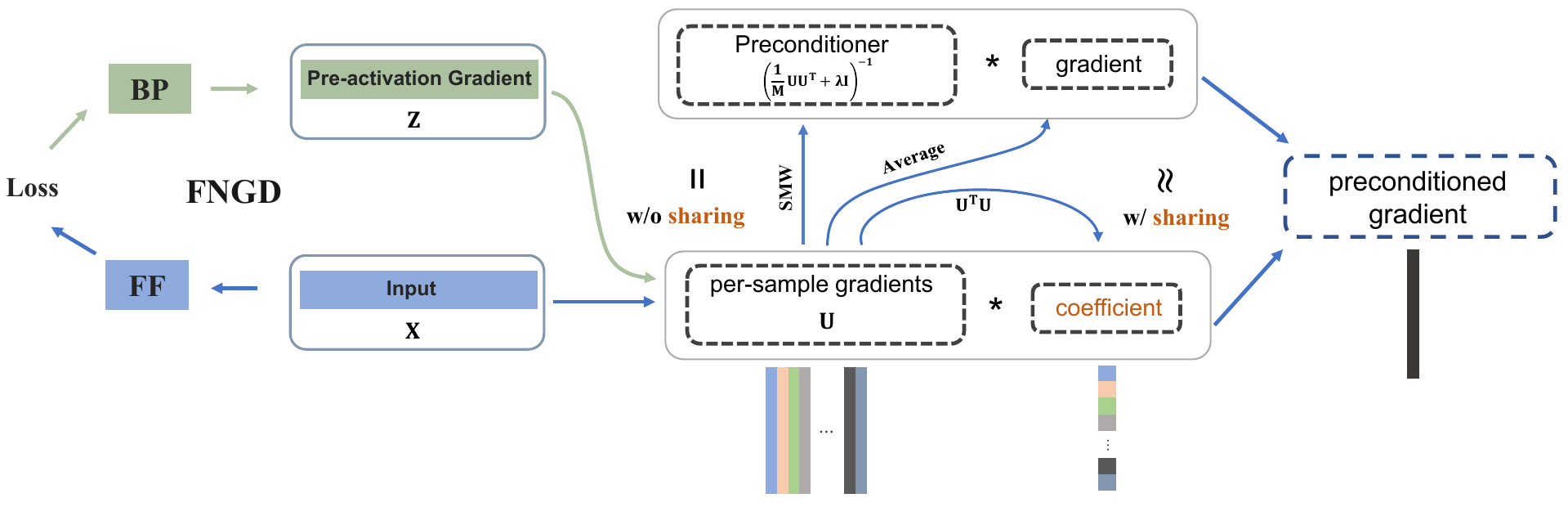}
    \caption{Illustration of FNGD. The gradient preconditioning formula in NGD can be equivalent to a weighted sum of per-sample gradients. By sharing these weighted coefficients across epochs, the proposed FNGD approximates the preconditioning step as a fixed-coefficient weighted sum. This approach reduces the computational complexity of FNGD to that of SGD.}
    \label{fig:framework}
\end{figure*}
First-order methods, such as stochastic gradient descent (SGD) \cite{robbins1951stochastic}, and its adaptive learning rate variants including AdaGrad \cite{duchi2011adaptive}, RMSprop \cite{hinton2012neural}, and Adam \cite{kingma2014adam}, are dominant for training deep neural networks (DNNs), especially in large-scale tasks. Nonetheless, these methods face significant challenges. One key issue is the slow convergence observed in flat regions or the oscillations observed in steep regions. Additionally, first-order methods are highly sensitive to hyperparameter configurations, often requiring extensive tuning efforts.


 Second-order methods offer a solution to these problems \cite{yao2021adahessian, battiti1992first, botev2017practical}. Newton's method \cite{nocedal1999numerical}, a conventional second-order approach, involves preconditioning gradient with the inverse of Hessian matrix. This preconditioning operation enables gradient to be rotated and rescaled, aiding in the escape from ill-conditioned 'valleys'. However, due to the non-convex nature of DNNs, Hessian matrix may not always be positive semi-definite (PSD). Consequently, in practical applications, second-order methods often rely on approximations of Hessian that ensure PSD properties, such as the Fisher information matrix (FIM) \cite{amari1998natural}. 


In the field of statistical machine learning, FIM is defined as the covariance of score function. Given a probabilistic model, the score function represents the gradient of log-likelihood with respect to model parameters $\mb w$. FIM, serving as the corresponding covariance matrix, \ie$\mathbb{E}(\nabla_{\mb{w}}\nabla^\text{T}_{\mb{w}})$, has been proved to be equal to the negative expected Hessian of the model’s log-likelihood \cite{ly2017tutorial}. Consequently, FIM can serve as an alternative to Hessian. However, in the case of deep learning with $N$ parameters, the FIM matrix size is $N \times N$, and the computational complexity for its inversion is $\operatorname{O}(N^3)$, rendering it impractical for DNNs with millions of parameters.


To bridge this gap, a series of algorithms \cite{martens2015optimizing, george2018fast, tang2021skfac} have been developed to approximate FIM by employing a block-diagonal structure, with each block corresponding to a specific layer. This block-wise approximation effectively eliminates inter-layer correlations. Furthermore, recognizing the low-rank nature of FIM, advanced methods \cite{ren2019efficient, tang2021skfac} utilize Sherman-Morrison-Woodbury (SMW) formula \cite{hager1989updating} to efficiently compute the matrix inverse. However, these techniques still necessitate the inverse operator in each epoch. As a result, the end-to-end training time of NGD may approach or even surpass that of SGD. 

To address this issue, this work presents a fast natural gradient descent (FNGD) method in which the inverse operator is performed exclusively during the first epoch. Firstly, we find that the gradient preconditioning formula in natural gradient descent (NGD) can be reformulated by SMW into a matrix-vector multiplication. It has the interpretation of a weighted sum of per-sample gradients. By re-arranging the computation order, we decrease the preconditioning computational complexity from $\operatorname{O}(N_lM^2+N_l^2M+N_l^2)$ to $\operatorname{O}(M^2+N_lM)$, where $N_l$ represents the number of parameters in layer $l$ and $M$ represents the batch size. Furthermore, it is analyzed that the weighted coefficient vector is solely determined by a correlation matrix that reveals the correlation of samples. This observation inspires us to share it across epochs, making FNGD akin to the average sum in SGD. As a result, the computational complexity of FNGD can approach that of SGD. We provide a complexity comparison between FNGD\footnote{We ignore the minor computational cost associated with computing coefficients in the first epoch.} and conventional second-order methods in \cref{tab:time}.

\begin{table}[htbp]
   
    \caption{Complexity comparison of different optimization algorithms. $D$ is the dimension of a hidden layer, and $M$ is the batch size.}
    \centering
    \begin{tabular}{c|c|c|c|c}
        \hline
         Method&  Statistics & Inverse & Precondition & Memory\\
         \hline
         KFAC \cite{martens2015optimizing} & $\operatorname{O}(2MD^2)$ & $\operatorname{O}(2D^3)$ &  $\operatorname{O}(2D^3)$ & $\operatorname{O}(2D^2)$\\
         \hline
         Eva \cite{zhang2023eva}&$\operatorname{O}(2MD)$ & - & $\operatorname{O}(2D^2)$ & $\operatorname{O}(2D)$\\
         \hline
         Shampoo \cite{gupta2018shampoo}&$\operatorname{O}(2D^3)$ & $\operatorname{O}(2D^3)$ & $\operatorname{O}(2D^3)$ & $\operatorname{O}(2D^2)$\\
         \hline
         FNGD &- & - & \boldmath{$\operatorname{O}(MD)$} & \boldmath{$\operatorname{O}(M)$}\\
         \hline
    \end{tabular}
     \label{tab:time}
\end{table}

We conduct numerical experiments on image classification and machine translation tasks to demonstrate the effectiveness and efficiency of FNGD. In the task of image classification, FNGD can yield comparable convergence and generalization performance to conventional second-order methods, such as KFAC \cite{martens2015optimizing}, Shampoo \cite{gupta2018shampoo}, and Eva \cite{zhang2023eva}. Additionally, FNGD surpasses these methods in terms of computational efficiency. Specifically, in comparison to KFAC, Shampoo, and Eva, FNGD can achieve time reductions of up to 2.65$\times$, 2.03$\times$, and 1.73$\times$, respectively.  In the context of machine translation with Transformer, FNGD outperforms AdamW by 24 BLEU score on Multi30K while requiring almost the same training time. Furthermore, when compared with other second-order methods, FNGD is approximately 2.4$\times$ faster than KFAC and 5.7$\times$ faster than Shampoo.


The main contributions of our work can be summarized as follows:
\begin{itemize}
    \item We reformulate the gradient preconditioning formula of NGD as a weighted sum of per-sample gradients using SMW formula. It establishes a connection between NGD and SGD. 
    \item We propose sharing these weighted coefficients across epochs by analyzing the physical meaning of coefficients. This approach eliminates the need to perform inverse operator for coefficients computation in every iteration, except in the initial epoch.
    \item An efficient preconditioning strategy is proposed to accelerate the implementation of FNGD, including efficient backward and weighted sum.
    \item we apply our approach for image classification and machine translation tasks to demonstrate its time efficiency.
\end{itemize}
\section{Related Work}
Natural gradient methods rely on FIM to precondition the gradient. To make it practical for DNNs, plenty of works apply a block-diagonal approximation to the FIM. In addition to this fundamental approximation,  various techniques have been introduced to further reduce the computational complexity. There are primarily two approaches: making further approximation or reducing the inverse complexity. In \cref{fig:approximation}, we outline several methods that offer further approximations of FIM, including KFAC \cite{martens2015optimizing}, Shampoo \cite{gupta2018shampoo}, Eva \cite{zhang2023eva}, and MBF \cite{bahamou2022mini}. To enhance the efficiency of computing the inverse, SKFAC \cite{tang2021skfac} proposed to employ the SMW formula to decrease the size of the matrix to be inverted. Furthermore, HyLo \cite{mu2022hylo} has the capability to decrease the matrix size by selecting important training samples, thereby improving the scalability of KFAC on distributed platforms.
\begin{figure}[htbp!]
    \centering
    \includegraphics[scale=0.4]{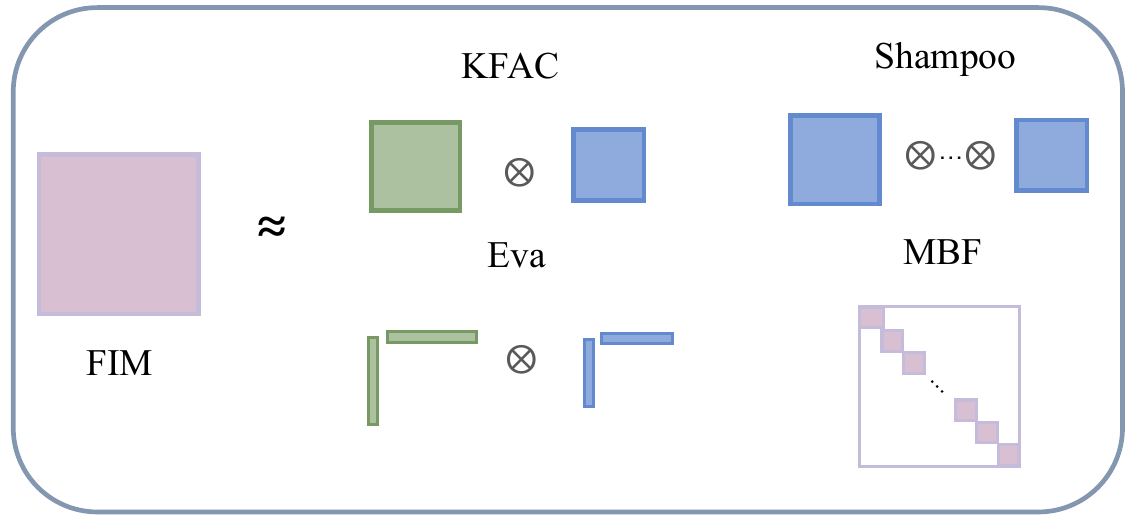}
    \caption{Several existing types of FIM approximation. The green block represents feed-forward statistics, while the blue block represents back-propagation statistics.}
    \label{fig:approximation}
\end{figure}

The Fisher approximation techniques discussed above primarily focus on computing the statistics or the inverse with low costs. However, these methods do not delve into analyzing the structure of preconditioned gradients. FNGD, on the other hand, strategically addresses this aspect to reduce the computational complexity effectively.


\section{Preliminaries}
\label{sec:formatting}
\subsection{Notation}
Scalars are denoted by letters, e.g., $a$. Vectors are denoted by boldface lowercase letters, e.g., $\mb{a}$. Matrices are denoted by boldface capital letters, e.g., $\mb{A}$. Higher-order tensors are denoted by Euler script letters, e.g., $\mathcal{X}$. The symbol $\odot$ signifies the Khatri-Rao product; $\otimes$ signifies the Kronecker product; $*$ signifies the Hadamard product; and $||\cdot||_\text{F}$ signifies the Frobenius norm of a matrix. The operator $\text{vec}(\cdot)$ reshapes a matrix into a vector by stacking rows.

\subsection{Second-order Method}
We consider a neural network $\operatorname{f}$ that maps the input data $\mb{x}_m$ with target $\mb{y}_m$ to an output prediction $\operatorname{f}(\mb{x}_m, \mb{w})$, where $\mb{w} \in \mathbb{R}^{N}$ consists of all parameters in the network. Training the network can be regarded as minimizing the cost function denoted by $\operatorname{\ell}$, e.g., cross-entropy. In each iteration with a batch of $M$ samples, the objective is formulated as $\mathcal{L}(\mb{w})=\frac{1}{M}\sum_{m=1}^{M}\operatorname{\ell}(\operatorname{f}(\mb{x}_m,\mb{w}), \mb{y}_m)$. For second-order methods, the parameter update can be formulated as follows:
\begin{equation}
    \mb{w} \leftarrow \mb{w} - \eta \mathbf{B}^{-1}\mb{g},
\end{equation}
where $\mb{g} = \frac{\partial \mathcal{L}}{\partial \mb{w}}$ represents the gradient of the objective with respect to $\mb{w}$, $\eta$ is a positive learning rate, and $\mb{B}$ involves curvature information of the loss landscape, named preconditioner. In the case of $\mb{B}=\mb{I}$, the second-order method degenerates into SGD. For the natural gradient descent method, the FIM is employed as the preconditioner.




\subsection{Natural Gradient Method}

FIM, an approximation to the Hessian, serves as the preconditioner in the natural gradient method.  To avoid the extra backward pass associated with FIM computation, EFM becomes a practical alternative to FIM \cite{bahamou2022mini}. The EFM is defined as:
\begin{equation}
    \mathbf{F} = \frac{1}{M}\sum_{m=1}^{M}\frac{\partial\operatorname{\ell}(\operatorname{f}(\mb{x}_m,\mb{w}),\mb{y}_m)}{\partial \mb{w}}(\frac{\partial \operatorname{\ell}(\operatorname{f}(\mb{x}_m,\mb{w}),\mb{y}_m)}{\partial \mb{w}})^\text{T}.
\end{equation}
The pairs $(\mb{x}_m, \mb{y}_m)$ are from the training dataset. Defining a matrix $\mathbf{U}=[\frac{\partial \operatorname{\ell}(\operatorname{f}(\mb{x}_1,\mb{w}),\mb{y}_1)}{\partial \mb{w}}, \frac{\partial \operatorname{\ell}(\operatorname{f}(\mb{x}_2,\mb{w}),\mb{y}_2)}{\partial \mb{w}}, \cdots, \frac{\partial \operatorname{\ell}(\operatorname{f}(\mb{x}_M,\mb{w}),\mb{y}_M)}{\partial \mb{w}}] \in \mathbb{R}^{N \times M}$, named the Jacobian matrix of the loss, the EFM can be represented in terms of $\mathbf{U}$ as follows:
\begin{equation}
    \mathbf{F} = \frac{1}{M}\mathbf{U}\mathbf{U}^\text{T}.
\end{equation}

In the context of deep learning, a mini-batch strategy is employed to alleviate the computational burden. This mini-batch approximation results in the low-rank characteristic of EFM \cite{ren2019efficient}. It necessitates the addition of $\lambda \mathbf{I}$ to ensure the invertibility of EFM (namely, the Levenberg-Marquardt (LM) method \cite{more2006levenberg}), where $\lambda$ is a damping parameter. The updating of network parameters is formulated as:
\begin{equation}\label{updata}
    \mb{w} \leftarrow \mb{w} - (\frac{1}{M}\mathbf{U}\mathbf{U}^\text{T}+\lambda \mathbf{I})^{-1} \mb{g}.
\end{equation}

Analogous to KFAC, we apply the block-diagonal approximation on the EFM. Consequently, the parameters of each layer can be updated separately. For layer $l$, we have the updating rule as follows:
\begin{equation}
     \mb{w}_l \leftarrow \mb{w}_l - (\frac{1}{M}\mathbf{U}_l\mathbf{U}_l^\text{T} + \lambda \mathbf{I})^{-1} \mb{g}_l,
\end{equation}
where $\mb{w}_l \in \mathbb{R}^{N_l}$ is the parameter of layer $l$, $\mb{g}_l\in \mathbb{R}^{N_l}$ is the gradient with respect to $\mb{w}_l$, and $\mathbf{U}_l\in \mathbb{R}^{N_l\times M}$.

\section{Proposed Method}

Firstly, based on the well-known SMW formula, we propose to restructure the updating formula in \cref{updata}. This adjustment allows for the interpretation of the preconditioned gradient as a weighted sum of per-sample gradients. By re-arranging computation order, we can decrease the preconditioning computational complexity from $\operatorname{O}(N_lM^2+N_l^2M+N_l^2)$ to $\operatorname{O}(M^2+N_lM)$. Furthermore, by sharing these weighted coefficients across epochs, we approximate the preconditioning step in NGD as a fixed-coefficient weighted sum, which closely resembles the average sum in SGD. Consequently, the theoretical complexity of FNGD is comparable to that of SGD. For the implementation, we provide a discussion on how to compute the per-sample gradient and perform the weighted sum efficiently.
The framework of FNGD is depicted in \cref{alg1}.
\begin{algorithm} 
\caption{Framework of FNGD} 
\label{alg1} 
\begin{algorithmic}[1] 
\REQUIRE Learning rate: $\eta$; Damping value: $\lambda$; Batch size: $M$; number of batches: $B$. 

\tcp{\footnotesize Derive coefficients during the first epoch}
\FOR{$i=1$ to $B$} 
\STATE Sample mini-batch of size $M$;
\STATE Perform forward-backward pass to compute $\mb U_i$;
\STATE Compute coefficients $\mb{v}_i$ based on \cref{coeffcient};
\STATE Perform parameter update: $\mb{w} \leftarrow \mb{w}-\eta \mb{U_i}\mb{v}_i$;
\ENDFOR
\STATE Compute average coefficient: $\Tilde{\mb{v}}=\frac{1}{B} \sum_{i=1}^B \mb{v}_i $;
\WHILE{not converage}
\FOR{$i=1$ to $B$}
\STATE Sample mini-batch of size $M$;
\STATE Perform forward-backward pass to derive gradients of layers' output $\mathcal{A}$;
\STATE Apply weighted coefficients on layers' inputs $\mathcal{B}$;
\STATE $\Tilde{\mb g}\leftarrow$ torch.nn.grad.layer\_weight($\mathcal{A}$, $\mathcal{B}$);
\STATE Perform parameter update: $\mb{w} \leftarrow \mb{w}-\eta \Tilde{\mb g}$;
\ENDFOR
\ENDWHILE
\end{algorithmic}
\end{algorithm}
\subsection{SMW-based NGD}
The SMW formula depicts how to efficiently compute the inverse of an invertible matrix perturbed by a low-rank matrix. Considering an invertible matrix $\mathbf{X}\in \mathbb{R}^{N\times N}$ and a rank-$K$ perturbation $\mathbf{A}\mathbf{B}$ with $\mathbf{A}\in\mathbb{R}^{N\times K}$ and $\mathbf{B}\in \mathbb{R}^{K\times N}$, the inverse of the matrix $\mb{X}+\mb{A}\mb{B}$ can be computed using $\mb{X}^{-1}$ as follows:
    \begin{equation}
         (\mathbf{X}+\mathbf{A}\mathbf{B})^{-1}=\mathbf{X}^{-1}-\mathbf{X}^{-1}\mathbf{A}(\mathbf{I}+\mathbf{B}\mathbf{X}^{-1}\mathbf{A})^{-1}\mathbf{B}\mathbf{X}^{-1}.
    \end{equation}

Based on the SMW formula and the low-rank property of $\mb{U}_l\mb{U}_l^\text{T}$, we can derive the inverse $(\lambda \mb{I}+\frac{1}{M}\mathbf{U}_l\mathbf{U}_l^\text{T})^{-1}$ as follows:
\begin{equation}
    (\lambda \mb{I}+\frac{1}{M}\mathbf{U}_l\mathbf{U}_l^\text{T})^{-1}=\frac{1}{\lambda}(\mathbf{I}-\frac{1}{M}\mathbf{U}_l(\lambda \mathbf{I}+\frac{1}{M}\mathbf{U}_l^\text{T}\mathbf{U}_l)^{-1}\mathbf{U}_l^\text{T}).
\end{equation} 
This approach reduces the size of the matrix to be inverted from $N_l\times N_l$ to $M\times M$. Therefore, as long as $M \ll N_l$ is satisfied, SMW-based NGD is much more favorable for devices with limited computational resources. 

The inverse is then utilized to precondition the gradient $\mb{g}_l$, as it does in \cite{ren2019efficient}. Assuming the calculation of $\mathbf{U}_l^\text{T}\mathbf{U}_l$ and $(\lambda \mathbf{I}+\frac{1}{M}\mathbf{U}_l^\text{T}\mathbf{U}_l)^{-1}$ have been completed, the remaining computational complexity is $\operatorname{O}(N_lM^2+N_l^2M+N_l^2)$. In order to decrease the complexity, we propose to re-arrange the multiplication order. Firstly, the preconditioning formula can be denoted as:
\begin{multline}
     (\lambda \mb{I}+\frac{1}{M}\mathbf{U}_l\mathbf{U}_l^\text{T})^{-1}\mb{g}_l = \\\frac{1}{\lambda}\mb{g}_l-\frac{1}{\lambda M}\mathbf{U}_l(\lambda \mathbf{I}+\frac{1}{m}\mathbf{U}_l^\text{T}\mathbf{U}_l)^{-1}\mathbf{U}_l^\text{T}\mb{g}_l.
\end{multline}
   
Furthermore, given that $\mb{g}_l$ represents the average gradient over the mini-batch, we can express $\mb{g}_l$ as the mean vector of the columns of matrix $\mathbf{U}_l$, \ie, $\mb{g}_l = \frac{1}{M}\mathbf{U}_l[1,1,\cdots,1]^\text{T}$. Building on this, we can reformulate the preconditioning equation as follows:
\begin{multline}
    (\lambda \mb{I}+\frac{1}{M}\mathbf{U}_l\mathbf{U}_l^\text{T})^{-1}\mb{g}_l=\\
    \frac{1}{\lambda M}\mathbf{U}_l([1,1,\cdots,1]^\text{T}-(\lambda \mathbf{I}+\frac{1}{M}\mathbf{U}_l^\text{T}\mathbf{U}_l)^{-1}\mathbf{U}_l^\text{T}\mb{g}_l).
    \label{coeffcient}
\end{multline}
It involves a matrix-vector multiplication. For the calculation of $\mathbf{U}_l^\text{T}\mb{g}_l$, it is equivalent to $\frac{1}{M}\mathbf{U}_l^\text{T}\mathbf{U}_l[1,1,\cdots,1]^\text{T}$, indicating that it equals the mean vector of the columns of $\mathbf{U}_l^\text{T}\mathbf{U}_l$. As a result, the computational burden mainly comes from two parts: the multiplication of the inverse with the mean vector, and the external matrix-vector multiplication. The preconditioning computational complexity is reduced to $\operatorname{O}(M^2+N_lM)$.

\subsection{Coefficient-Sharing}
In \cref{coeffcient}, we represent the preconditioning equation as a matrix-vector multiplication, highlighting that the preconditioned gradient is a weighted sum of per-sample gradients. Building upon this reformulation, we propose reasonably sharing the coefficient across iterations to avoid the computationally intensive inverse operation.

 Firstly, we can see that the weighted coefficient vector is solely determined by the matrix $\mathbf{U}_l^\text{T}\mathbf{U}_l$. The matrix $\mathbf{U}_l^\text{T}\mathbf{U}_l$ is, in essence, a Gram matrix, where each entry reveals the similarity between two training samples. Specifically, when two samples are distinct, their gradient directions might be orthogonal within the parameter space, leading to a near-zero entry in $\mathbf{U}_l^\text{T}\mathbf{U}_l$. In contrast, gradients for similar samples are likely to closely align in direction, yielding a significant entry in $\mathbf{U}_l^\text{T}\mathbf{U}_l$.

For each layer, we have an individual correlation matrix $\mathbf{U}_l^\text{T}\mathbf{U}_l$. This setting aligns with the concept of ``hierarchical feature learning" in deep learning \cite{goodfellow2016deep}. As training samples pass through the deep network, lower layers tend to capture basic features, while higher layers capture more abstract and complex features. Consequently, the correlation matrix of the higher layer is expected to differ significantly from that of the lower layer, whereas the adjacent layers may exhibit similar correlation matrices. We depict this point in \cref{fig:UTU}.
\begin{figure}[htbp]
    \centering
    \includegraphics[scale=0.55]{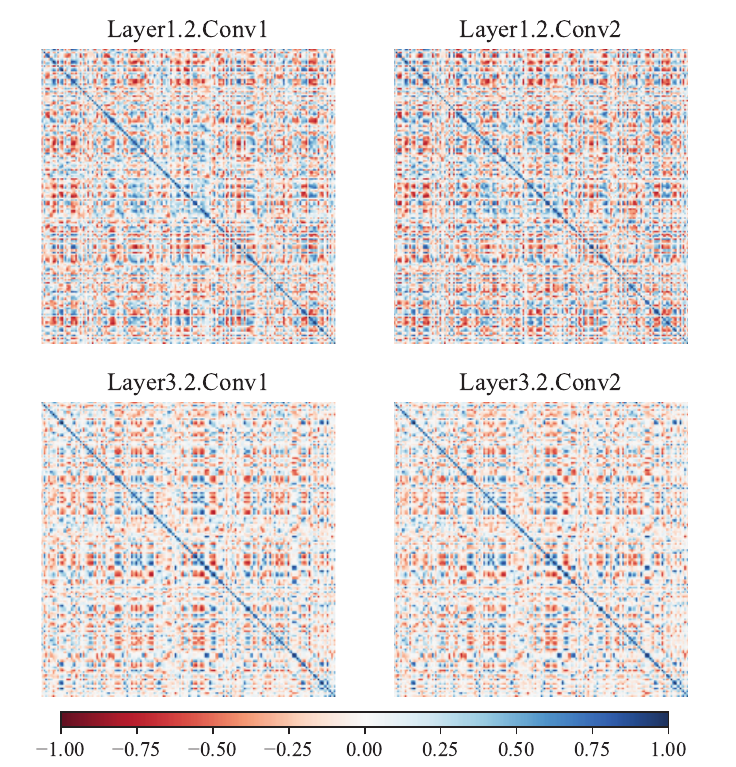}
    \caption{The normalized correlation matrix $\mb{U}_l^\text{T}\mb{U}_l$ for four layers in ResNet-32 \cite{he2016deep} on CIFAR-10 with batch size 128. The first two layers are adjacent, as are the last two layers.}
    \label{fig:UTU}
\end{figure}

 From an optimization perspective, the gradient weighted sum technique is an analogy to the weighted loss concept \cite{zadrozny2003cost,lin2017focal}, which is widely employed in machine learning to direct the optimization's attention. With fixed weighting coefficients, samples are assigned varying degrees of importance. However, a unique aspect of FNGD lies in its layer-specific optimization approach, where each layer is assigned distinct weighting coefficients. This characteristic might contribute to layer-specific feature learning.

Rather than randomly set coefficients, FNGD derives the coefficients of each layer by initially computing \cref{coeffcient} during the first epoch. By doing so, the coefficients of adjacent layers could be coupled. This level of interdependence among coefficients wouldn't be achieved through random initialization alone.

In practical applications, considering the mini-batch strategy, the weighted sum is performed within each mini-batch. Routinely, in order to enhance the generalization performance, the training dataset is shuffled before being divided into batches in each epoch. This shuffling operator randomizes the samples within each batch. However, our empirical experiment has found that coefficient-sharing across epochs remains effective despite the variability in samples within each batch. 

In \cref{fig:sharing}, we demonstrate the comparative results (with or without coefficient-sharing) for ResNet-32 on the Cifar-10 dataset. Remarkably, FNGD, the one with coefficient-sharing, achieves performance that is on par with the baseline NGD. On the other hand, thanks to the technique of coefficient-sharing, there is no need to compute the second-order information for epochs beyond the first epoch. Consequently, this results in a significant reduction in time costs. Specifically, FNGD is shown to be twice as fast as NGD.
\begin{figure}[htbp]
    \centering
    \includegraphics[scale=0.25]{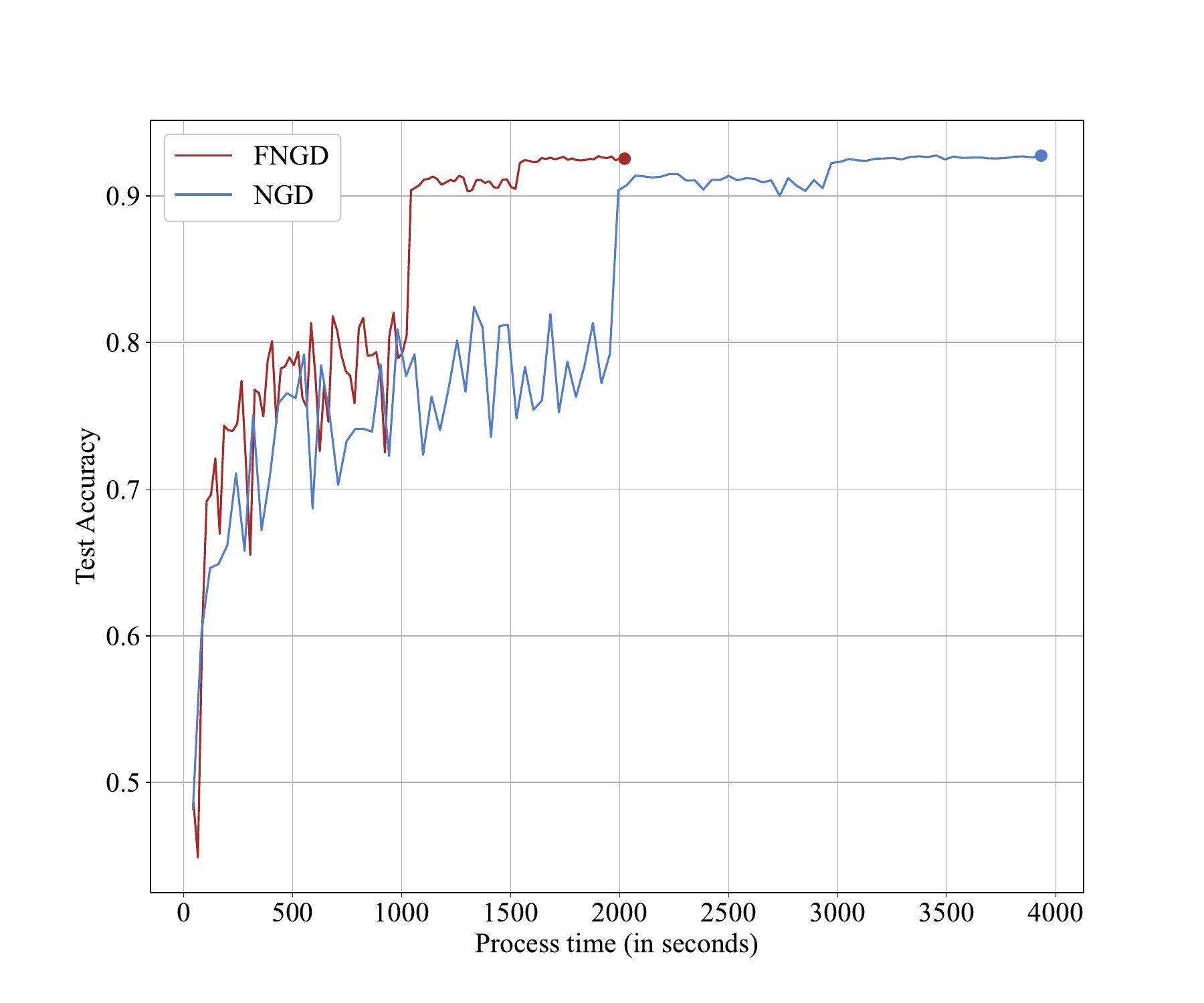}
    \caption{Performance comparison between FNGD and NGD for training ResNet-32 on Cifar-10. We refer to the method with coefficient-sharing as FNGD.}
    \label{fig:sharing}
\end{figure}

To clarify the effectiveness of FNGD, we can interpret the impact of shuffling-induced randomness from a different perspective. 
A sample with a large coefficient significantly influences the update process, while one with a small coefficient has a minor impact. As a result, with coefficients constant, FNGD randomly picks certain samples as key contributors to guide the optimization process. It may increase the model's robustness to noise.

\subsection{Per-sample Gradient} \label{pre-sample graident}
It is crucial to efficiently calculate the per-sample gradient for the computational efficiency of FNGD. Popular deep learning frameworks, like Pytorch and Tensorflow, return the average gradient over a batch of samples, rather than the gradient for each individual sample. This choice is primarily for memory efficiency. Per-sample gradient computation has been discussed in the context of differential privacy (DP). Opacus \cite{opacus}, a popular Python library for training DNNs with DP, obtains the per-sample gradient based on module hooks. Module hooks are a mechanism in Pytorch designed for capturing individual modules' features, including input, output, and gradients. 

Although hooks allow us to compute the per-sample gradient effectively through vectorized computation, they are triggered by the calculation of parameter gradients. That is to say, before deriving the per-sample gradient, the average gradient is needlessly computed. To eliminate this redundancy, we propose to make use of Autograd to compute the gradient of modules' output, instead of modules' parameters. Then together with the reserved modules' input, we can derive the pre-sample gradient.

Considering a fully connected layer with input $\mb{X}\in \mathbb{R}^{I\times M}$ and weight $\mb{W}\in \mathbb{R}^{O\times I}$, we have the pre-activation output $\mb{Y}=\mb{W}\mb{X}\in \mathbb{R}^{O\times M}$, whose gradient is denoted as $\mb{Z}\in \mathbb{R}^{O\times M}$. With the reserved input $\mb{X}$ and the gradient $\mb{Z}$, we can derive the gradient of $\mb{W}$ for sample $m$ as follows:
\begin{equation}
    \mb{G}^m = \mb{Z}_{:,m}\mb{X}_{:,m}^\text{T},
\end{equation}
where $\mb{X}_{:,m}$ and $\mb{Z}_{:,m}$ are the $m$-th column of $\mb{X}$ and $\mb{Z}$, respectively. The vectorized form is $\text{vec}(\mb{G}^m)=\mb{Z}_{:,m}\otimes\mb{X}_{:,m}$. Therefore, we can represent $\mb{U}_l$ as the Khatri-Rao product of $\mb{Z}$ and $\mb{X}$ as follows:

\begin{equation}
    \mb{U}_l =  \mb{Z}\odot\mb{X}.
\end{equation}
For computing $\mb{U}_l^\text{T}\mb{U}_l$, we can make use of the identity $(\mb{A}\odot \mb{B})^\text{T}(\mb{A} \odot \mb{B})= \mb{A}^\text{T}\mb{A}*\mb{B}^\text{T}\mb{B}$ to get
\begin{equation}
    \mb{U}_l^\text{T}\mb{U}_l =\mb{Z}^\text{T}\mb{Z}*\mb{X}^\text{T} \mb{X},
    \label{trans}
\end{equation}
which can decrease the computational complexity of computing the correlation matrix from $\operatorname{O}(O^2I^2M+OIM)$ to $\operatorname{O}(O^2M+I^2M+M^2)$. It can significantly decrease the computation burden for wide fully connected layers.

For a conventional layer with padded input patches $\mathcal{X}\in \mathbb{R}^{IK^2\times S \times M}$ and weight $\mathbf{W}\in \mathbb{R}^{O\times IK^2}$, where $I$, $O$, $S$, $K$ denotes the sizes of the input channel, output channel, patches, and kernel size, respectively, we have the output $\mathcal{Y}\in \mathbb{R}^{O\times S\times M}$ with $\mathcal{Y}_{o,s,m}=\sum_{i=1}^{IK^2}\mathcal{X}_{i,s,m}\mb{W}_{o,i}$. We can derive the equation of $\mb{U}_l$ for convolutional layers as follows:
\begin{equation}
    \mb{U}_l = \sum_{s=1}^{S} \mathcal{Z}_{:,s,:} \odot \mathcal{X}_{:,s,:},
\end{equation}
where $\mathcal{Z}$ is the gradient of the output patches $\mathcal{Y}$. Due to the summation operator, we can't employ the identity equation to reconstruct $\mb{U}_l^\text{T}\mb{U}_l$ as \cref{trans}. Nonetheless, as the number of parameters in convolutional layers is much smaller than that of fully connected layers, the computation of $\mb{U}_l^\text{T}\mb{U}_l$ is typically affordable.

After deriving these coefficients, to mitigate the significant memory overhead associated with storing $\mb{U}_l$, we propose to apply weighted coefficients directly on inputs, and then we can derive the preconditioned gradient through the gradient computation tools in PyTorch, by using weighted inputs and output gradients. This approach also cancels out the time-consuming patch-extracting operator. Therefore, the additional computational complexity introduced by FNGD over SGD is primarily due to the input weighting step with $\operatorname{O}(MD)$ complexity as listed in \cref{tab:time}.

\subsection{Setting of Damping} \label{set_damping}
In \cref{coeffcient}, the addition of $\lambda \mb{I}$ to the low-rank matrix $\mb{U}_l^\text{T}\mb{U}_l$ serves to ensure the invertibility. Simultaneously, the term $\frac{1}{\lambda}$ is multiplied to scale the coefficients vector. In essence, the choice of the damping parameter $\lambda$ will markedly impact the performance of optimization.

Firstly, the value of $\lambda$ has an influence on the inverse $(\lambda \mb{I}+\mb{U}_l^\text{T}\mb{U}_l)^{-1}$. A small $\lambda$ may give rise to issues of numerical instability, whereas an excessively large $\lambda$ may lead to a degradation in the inverse precision. In order to appropriately determine $\lambda$, we establish a proportionality between $\lambda$ and the Frobenius norm of $\mathbf{U}_l^\text{T}\mathbf{U}_l$, which can be formulated as follows:
\begin{equation}
    \label{prop}
    \lambda = \alpha ||\mb{U}_l^\text{T}\mb{U}_l||_\text{F}.
\end{equation}

Moreover, for the ease of tuning $\lambda$, we incorporate the scaling factor $\frac{1}{\lambda}$ into the learning rate. Consequently, we only need to consider the impact of $\lambda$ on the remaining portion, \ie $\frac{1}{M}([1,1,\cdots,1]^\text{T}-(\lambda \mathbf{I}+\frac{1}{M}\mathbf{U}_l^\text{T}\mathbf{U}_l)^{-1}\mathbf{U}_l^\text{T}g_l)$. The tuning principle we employ is to choose $\alpha$ such that the remaining portion approximates $\frac{1}{M}$. Through this strategy, FNGD is akin to SGD but with fluctuating weighted coefficients. On the other hand, the step size is now changed from $\eta$ to $\frac{\eta}{\lambda}$. As the $\lambda$ is related to the second-order moment of gradients, we can view the step size as an adaptive learning rate, analogous to Adam. It may have the potential to speed up the convergence.

\section{Convergence}
We follow the framework established in \cite{zhang2019fast} to provide a theoretical convergence analysis. For simplicity, we consider the single-output case with squared error loss, while the multiple-output case can be extended similarly. Furthermore, we focus on the full-batch case.

Given a dataset $\{\mb{x}_m, y_m\}_{i=1}^{M}$ with $\mb{x}_m\in \mathbb{R}^{d} $ and $y_m\in \mathbb{R}$, we have 
the network output vector denoted as $\mb{v}(\mb w) = [\operatorname{f}(\mb w, \mb{x}_1), \cdots, \operatorname{f}(\mb w, \mb{x}_M))]^{\text{T}}$, and the squared error loss can be denoted as:
\begin{align*}
    \min_{\mb{w}} \frac{1}{2}\|\mb v(\mb{w})-\mb{y} \|^2_2,
\end{align*}
where $\mb{y}=[y_1, y_2, \cdots, y_M]^{\text{T}}$. We denote the Jacobian matrix as $\mb J(\mb w) = [\frac{\partial \operatorname{f}(\mb{w}, \mb{x}_1) }{\partial \mb w}, \cdots, \frac{\partial \operatorname{f}(\mb{w}, \mb{x}_M) }{\partial \mb w}]\in \mathbb{R}^{N\times M}$. Consequently, the averaged gradient vector $\mb{g}(\mb w)\in \mathbb{R}^{N}$ is given by $\frac{1}{M}\mb J (\mb w)(\mb{v}(\mb w)-\mb y)$.

In FNGD, the fisher information matrix is block-diagonal approximated, \ie $\mb F = \frac{1}{M}\Tilde{\mb{J}}\Tilde{\mb{J}}^{\text{T}}$, where $\Tilde{\mb{J}}\in\mathbb{R}^{N\times ML}$ is a variant of $\mb J$. This variant is defined as $\mb{J} = \Tilde{\mb{J}}\mb{K}$, with $\mb{K}=[\mb{I}_M, \mb{I}_M, \cdots, \mb{I}_M]^{\text{T}}\in\mathbb{R}^{ML\times M}$. Here, $L$ denotes the number of network layers.

We make two assumptions that suffice to prove the global convergence of FNGD. The first assumption is in line with the coefficient-sharing technique in FNGD. The second assumption requires that $\mb J$ is stable enough for small perturbations around the initialization, which guarantees that the network is close to a linearized network.

\begin{assumption}\label{as1}
    The Gram matrix $\mb G = \Tilde{\mb{J}}^{\text{T}}\Tilde{\mb{J}}$ is constant across iterations, and $\mb G$ is a positive definite matrix, with $0 <\lambda_{\min}\leq\operatorname{\lambda}(\mb G)\leq \lambda_{\max}< +\infty$.
\end{assumption}

\begin{assumption}\label{as2}
    For all parameters $\mb w$ that satisfy $\| \mb{w} - \mb{w}_0 \|_2 \leq \frac{\sqrt{\lambda_{\max}L}}{\lambda_{\min}}\|\mb y-\mb{v}(\mb{w}_0)\|_2 $, we have $\|\mb {J}(\mb{w})-\mb{J}(\mb{w}_0)\|_2 \leq \frac{\lambda_{\min}}{2\sqrt{\lambda_{\max}}}$.
\end{assumption}

\begin{theorem}\label{th1}
    Let Assumptions  \ref{as1} and \ref{as2} hold. Suppose we optimize with FNGD using a damping value $\lambda = \frac{\lambda_{\min}}{M}$ and a small enough learning rate $\eta \leq \Tilde{\eta}$, we have $\|\mb{v}_k - y\|_2^2 \leq (1-\eta)^k \|\mb{v}_0 - y\|_2^2$.
\end{theorem} 

The detailed proof can be found in the Appendix. Although the optimization landscape is non-convex, FNGD with a constant learning rate and damping value enjoys at least a linear convergence rate. With a larger minimum eigenvalue $\lambda_{min}$, it is feasible to employ a larger learning rate so that FNGD can converge within fewer iterations.
\section{Experiments}
In this section, we compare FNGD with prevailing first-order methods, such as SGD and AdamW \cite{loshchilov2017decoupled}, as well as second-order methods, like KFAC, shampoo, and Eva. We examine their performance on the following two tasks: image classification and machine translation. Each algorithm was executed with the optimal hyperparameters determined through a grid search. In the case of KFAC and Shampoo, we set the frequency for updating second-order statistics to $T_1=10$ and the frequency for inverting to $T_2=100$. For Eva, we follow the setting described in \cite{zhang2023eva} to update the second-order statistics at each iteration. 

For all the algorithms mentioned above, we only utilize second-order statistics to precondition the gradient of convolutional layers and fully connected layers. For other trainable layers including BatchNorm layers, LayerNorm layers, and Embedding layers, we directly adopt the gradient descent direction. When implementing the KFAC algorithm, we follow the suggestion in \cite{pauloski2020convolutional, grosse2016kronecker} that employs eigenvalue decomposition on Kronecker factors to compute the inverse, which has been shown to yield higher test accuracy compared to directly inverting. For the implementation of Shampoo and Eva, we use the publicly available code\footnote{\href{https://github.com/lzhangbv/eva}{https://github.com/lzhangbv/eva}}. Our experiments were run on GeForce RTX 3080 GPUs using PyTorch.

\subsection{Image Classification}
We first evaluate our method's effectiveness and time efficiency on image classification tasks. To evaluate our method across Convolutional Neural Networks (CNNs) with varying widths and depths, we run experiments using five models: ResNet-32 and VGG-11 on CIFAR-10 dataset, ResNet-18 and ResNet-34 on CIFAR-100 dataset, and ResNet-50 on large-scale ImageNet dataset. In our experiments of CIFAR-10 and CIFAR-100, the first-order method SGD with momentum 0.9 (SGD-m) was run for 200 epochs, while second-order methods were run for 100 epochs. We decay the learning rate by a factor of 0.1 at 50\% and 75\% of the training epochs. Regarding ImageNet, the number of epochs was set to 90 for SGD-m and 55 for second-order methods. We set a batch size of 128 for all algorithms. 

We present the optimization curves of FNGD and the other mentioned algorithms for CIFAR-10 and CIFAR-100 in \cref{cifar10} and \cref{cifar100}, respectively. One can see that, for the four image classification tasks, FNGD can achieve comparable convergence and generalization performance when compared to other second-order methods. In terms of time efficiency, FNGD outperforms the other algorithms. It is noticeable that in the majority of cases, KFAC tends to be slower than SGD-m. In contrast, FNGD achieves comparable or superior accuracy to SGD-m in merely half the time.

This time efficiency contributes to the efficient training of FNGD. In \cref{tab:statistical}, we take the ResNet-32 on CIFAR-10 as an example to demonstrate the optimization advantage of FNGD. As indicated in \cref{tab:statistical}, FNGD achieves a 32.6\% reduction in training time compared to SGD-m to attain the target accuracy of 92.5\%. In contrast, Eva and Shampoo exhibit increase in training time of 17.3\% and 1.7\%, respectively.

\begin{figure}[htbp]
     \subfloat[VGG-11]{
    \centering
    \includegraphics[scale=0.3]{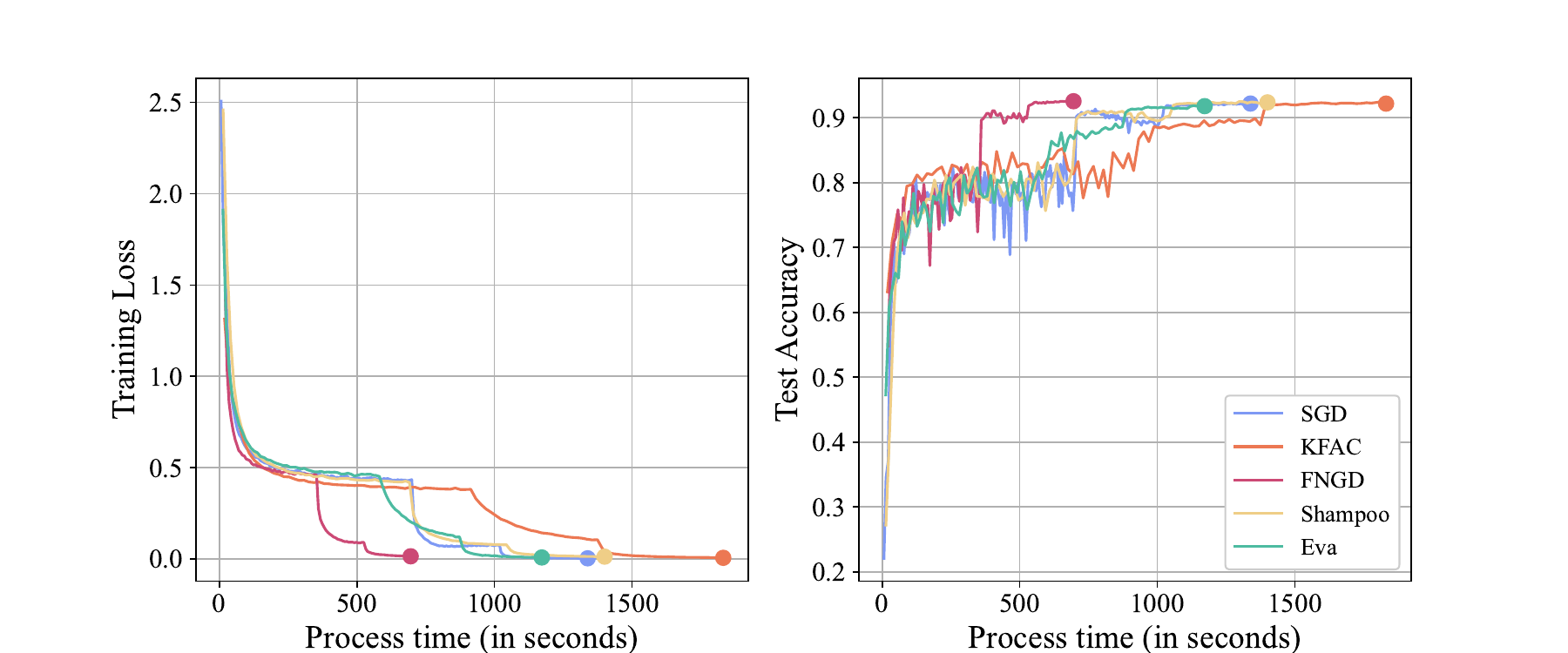}
    }
    
    \subfloat[ResNet-32]{
    \centering
    \includegraphics[scale=0.3]{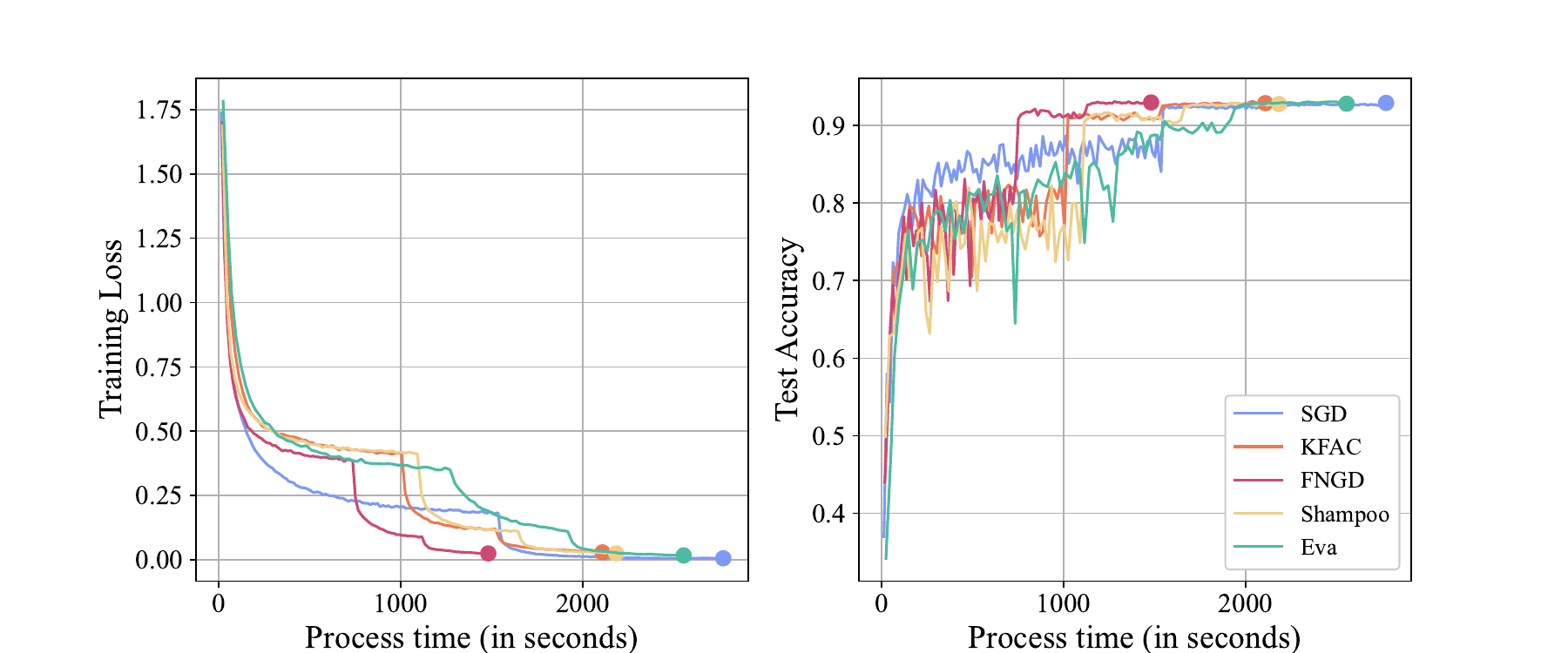}
    }
    
    \caption{The optimization curves of FNGD, SGD-m, KFAC, Shampoo, and Eva on ResNet-32 and VGG-11 with the CIFAR-10 dataset.}
    \label{cifar10}
\end{figure}

\begin{figure}[htbp]
    \subfloat[ResNet-18]{
    \centering
    \includegraphics[scale=0.3]{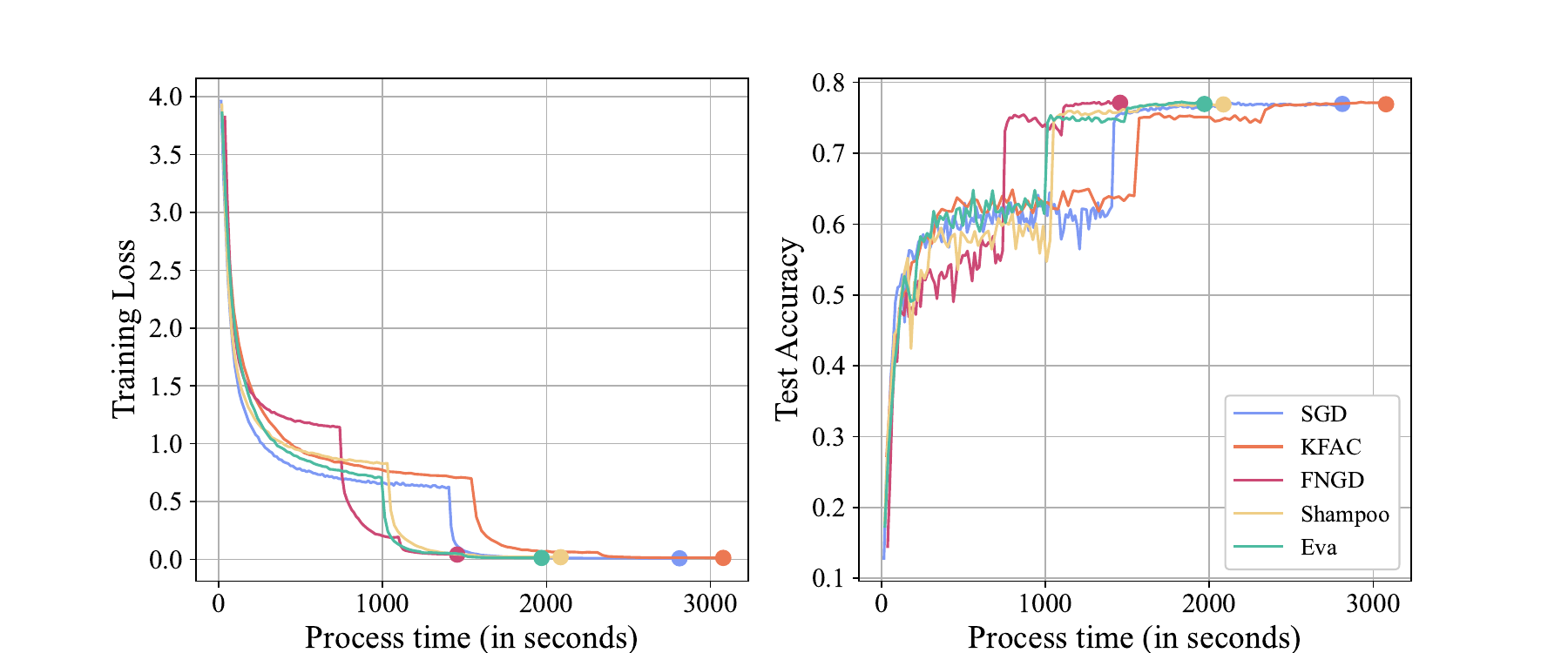}
    }
    
    \subfloat[ResNet-34]{
    \centering
    \includegraphics[scale=0.3]{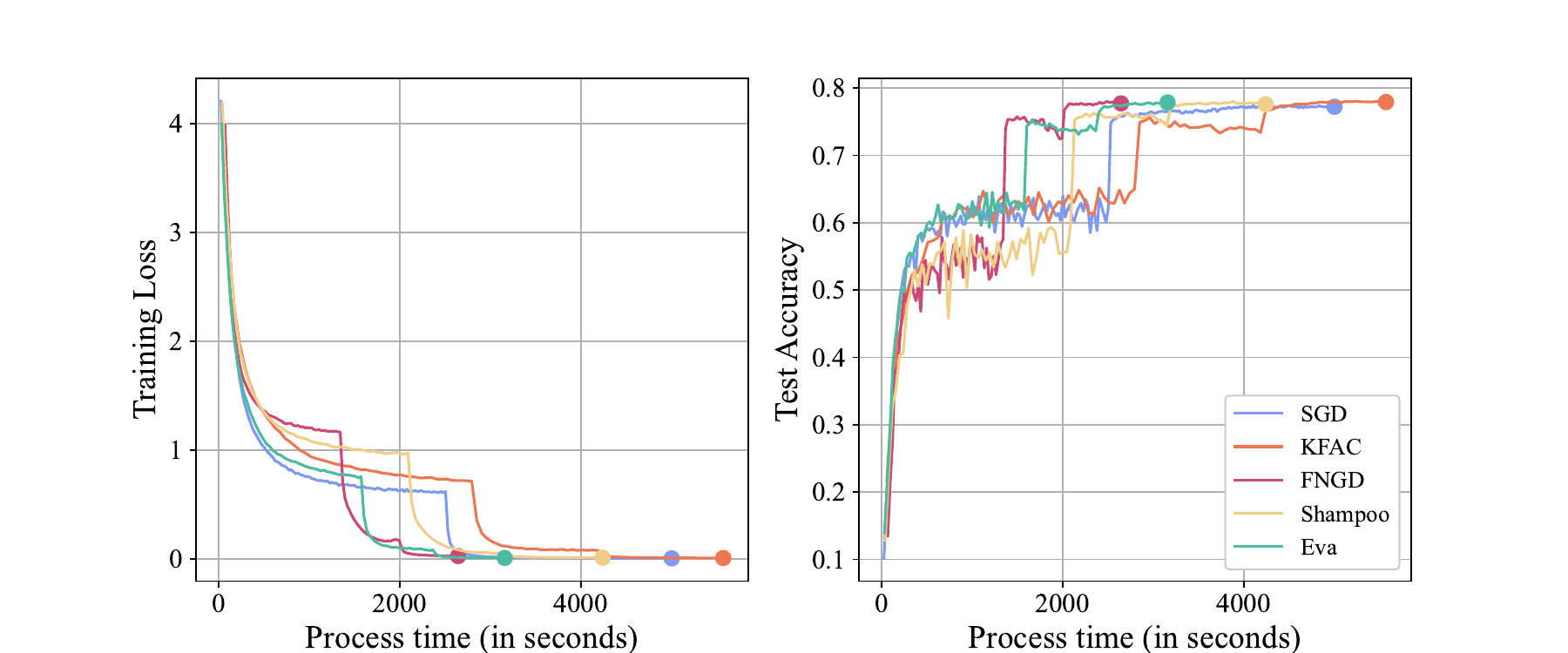}
    }
    \caption{The optimization curves of FNGD, SGD-m, KFAC, Shampoo, and Eva on ResNet-18 and ResNet-34 with the CIFAR-100 dataset.}
    \label{cifar100}
\end{figure}

\begin{table}[htbp]
\caption{Comparison of FNGD and other methods in terms of running time when reaching the target accuracy of 92.5\% for ResNet-32 on CIFAR-10.}
\resizebox{0.5\textwidth}{!}{
\begin{tabular}{c|c|c|c|c|c}
\hline
Method    & SGD-m & KFAC    & Shampoo  & Eva &FNGD    \\
\hline
Epoch     & 112   & 75      & 77        & 76 & 75      \\
Time (s) & 1677  & 1542    & 1706    & 1967 & 1131    \\
Time Gap  & 0\%   & -8.1\% & +1.7\%  & +17.3\% &\textbf{-32.6\%} \\
\hline
\end{tabular}}
\label{tab:statistical}
\end{table}
\begin{table}[]
        \centering
        \caption{Comparison of per-epoch training time between FNGD and other algorithms.}
        \resizebox{0.5\textwidth}{!}{
        \begin{tabular}{c|c|c|c|c|c|c}
        \hline
        Dataset & Model & SGD-m & KFAC & Shampoo & Eva & FNGD \\
        \hline
        \multirow{2}{*}{CIFAR-10} & VGG-11 & 1$\times$&2.73$\times$&2.09$\times$&1.75$\times$&\textbf{1.03$\times$}\\ &ResNet-32&1$\times$&1.52$\times$&1.57$\times$&1.84$\times$&\textbf{1.06$\times$}\\
        \hline
        \multirow{2}{*}{CIFAR-100} & ResNet-18 &1$\times$&2.18$\times$&1.48$\times$&1.40$\times$&\textbf{1.03$\times$} \\
    &ResNet-34&1$\times$&2.22$\times$&1.69$\times$&1.26$\times$&\textbf{1.05$\times$}\\
        \hline
        \end{tabular}
        \label{time efficiency}}
    \end{table}

We give a detailed time efficiency comparison of FNGD and other algorithms in \cref{time efficiency}. It is evident that FNGD exhibits the shortest per-epoch training time among all the evaluated second-order methods, nearly matching the performance of SGD-m. On average, the per-epoch training time of FNGD is 1.04$\times$ longer than that of SGD. When compared to KFAC, Shampoo, and Eva, FNGD can achieve speedup factors of up to 2.07$\times$, 1.64$\times$, and 1.50$\times$, respectively. Note that the relative time cost of Eva is higher than what is reported in \cite{zhang2023eva}. This is because the batch size we utilize is much smaller than the 1024 mentioned in \cite{zhang2023eva}, which results in more iterations per epoch. Consequently, there will be more statistical information computations and preconditioning operators.

For the larger-scale dataset ImageNet, we evaluate FNGD's performance on ResNet50 as shown in \cref{fig:imagenet}. It is clear that the training time of FNGD is nearly half of that of SGD, and FNGD is the fastest algorithm to achieve the target accuracy of 75\%. For the distributed implementation of FNGD, we can first perform the weighted sum of per-sample gradients on each server. Subsequently,  these distributed preconditioned gradients will be aggregated. In this way, there are no additional communication costs introduced by FNGD.
\begin{figure}
    \centering
    \includegraphics[scale=0.3]{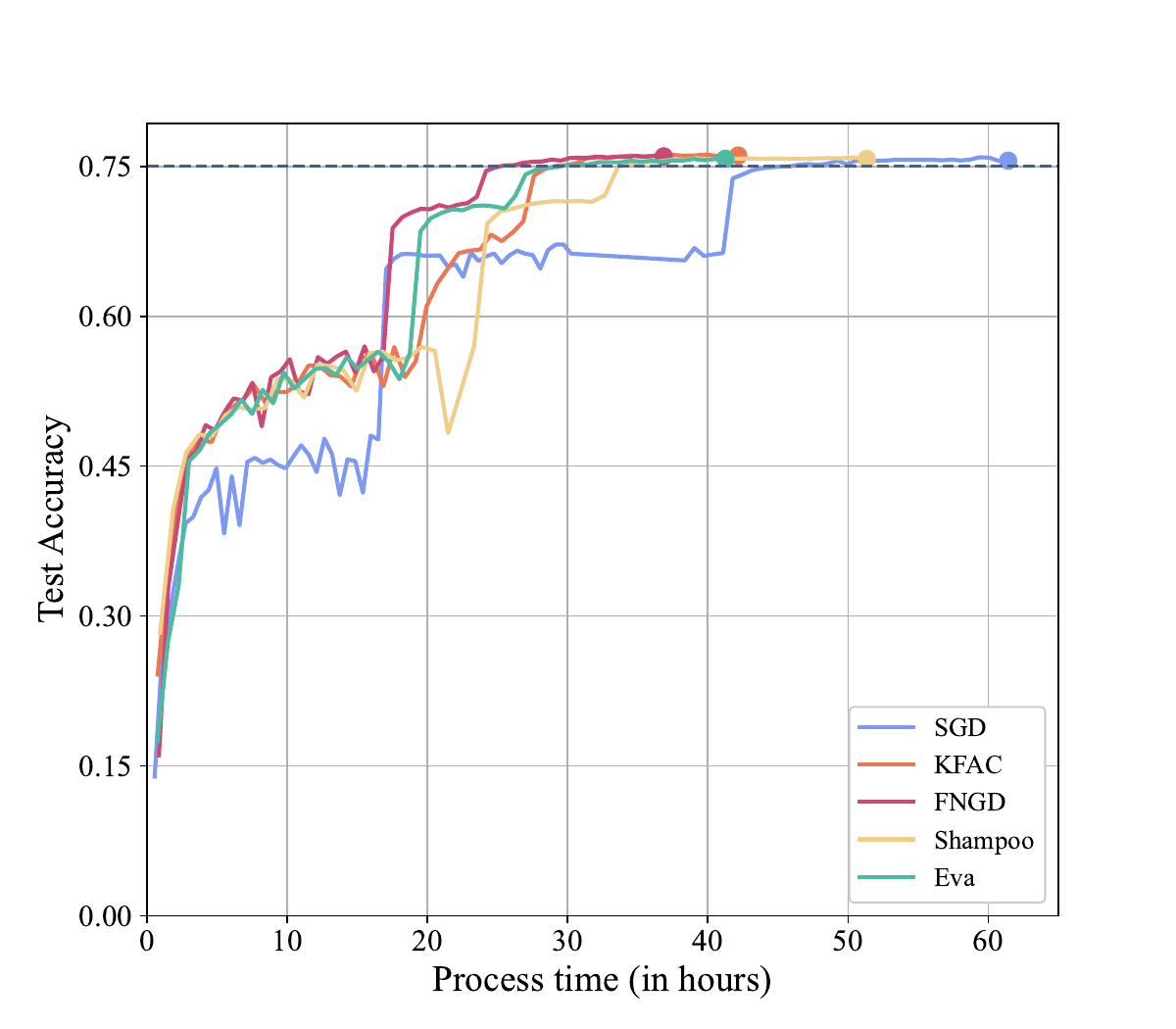}
    \caption{Generalization performance for ResNet-50 on ImageNet.}
    \label{fig:imagenet}
\end{figure}


\subsection{Machine Translation}
In the context of machine translation, we examine the efficiency of FNGD using the Transformer model with the Multi30K dataset. The Multi30K comprises image descriptions in both English and German. We adopt the conventional Transformer architecture described in \cite{vaswani2017attention}. Each block in the Transformer is configured with a model dimension of 512, a hidden dimension of 2048, and 8 attention heads. We utilize the metric BLEU to evaluate the quality of machine translation.

For natural language processing tasks, SGD performs much worse than AdamW as demonstrated in \cite{yao2021adahessian}. Therefore, we included comparative experiments with AdamW. We didn't include Eva in our experiments as its effectiveness for the Transformer has not been confirmed in \cite{zhang2023eva}. We run all the algorithms for 100 epochs with a batch size of 64. 

Our results are shown in \cref{fig:transformer}. It is demonstrated that FNGD yields the highest BLEU score on the test dataset. Specifically, the BLEU score achieved by FNGD surpasses that of SGD-m by nearly 35 points, and FNGD outperforms AdamW by 24 points, Shampoo by 11 points, and KFAC by 4 points. Furthermore, FNGD outperforms these second-order methods in terms of end-to-end training time. FNGD can achieve a similar time cost as both SGD-m and AdamW. In comparison to KFAC, FNGD is approximately 2.4$\times$ faster. Moreover, there is a remarkable time gap between Shampoo and FNGD, which differs from the situation with CNN models. This is attributed to the fact that the dimensions of layers in the Transformer are no less than 512, which significantly increases the computational complexity associated with inverting.


\begin{figure}
    \centering
    \includegraphics[scale=0.3]{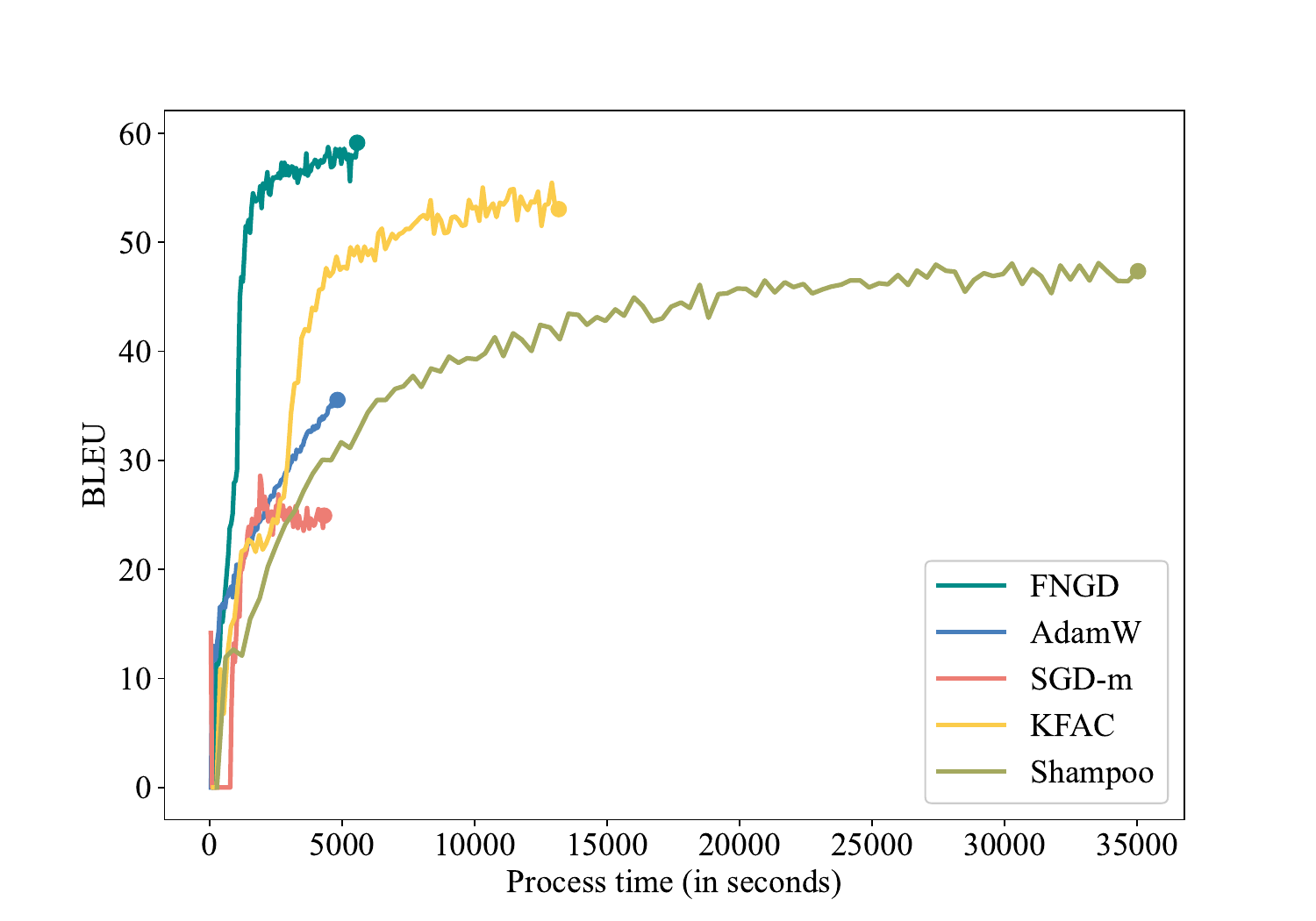}
    \caption{Test BLEU scores of Transformer on Multi30K using FNGD, AdamW, KFAC, and Shampoo.}
    \label{fig:transformer}
\end{figure}

\subsection{Time Analysis}
In order to have a thorough understanding of the time cost of FNGD, we provide a detailed analysis of the time cost of each step within FNGD in \cref{time_analysis}. We present the analysis on two different types of networks, \ie, ResNet and Transformer. As in SGD-m, the training process involves three primary steps: forward pass, backward pass, and parameter update. The FNGD and AdamW have an additional preconditioning step. 

In \cref{time_analysis}, we can see that the backward time in FNGD is less than the standard time cost in SGD-m. This results from our strategy of efficiently computing per-sample gradients. As mentioned in \cref{pre-sample graident}, we propose to make use of Autograd to compute the gradient of modules' output. It can be seen as a substep of the SGD-m backward process.

The total time spent on the backward and preconditioning processes of FNGD is comparable to the backward time of SGD-m, resulting in the overall computational complexity of FNGD approaches that of SGD-m. For the wider Transformer structure, its total time slightly exceeds that of the standard backward, as the computational cost of the preconditioning step is related to the layer dimension.

\begin{figure}[htbp!]
    \subfloat[ResNet-32]{
    \centering
    \includegraphics[scale=0.19]{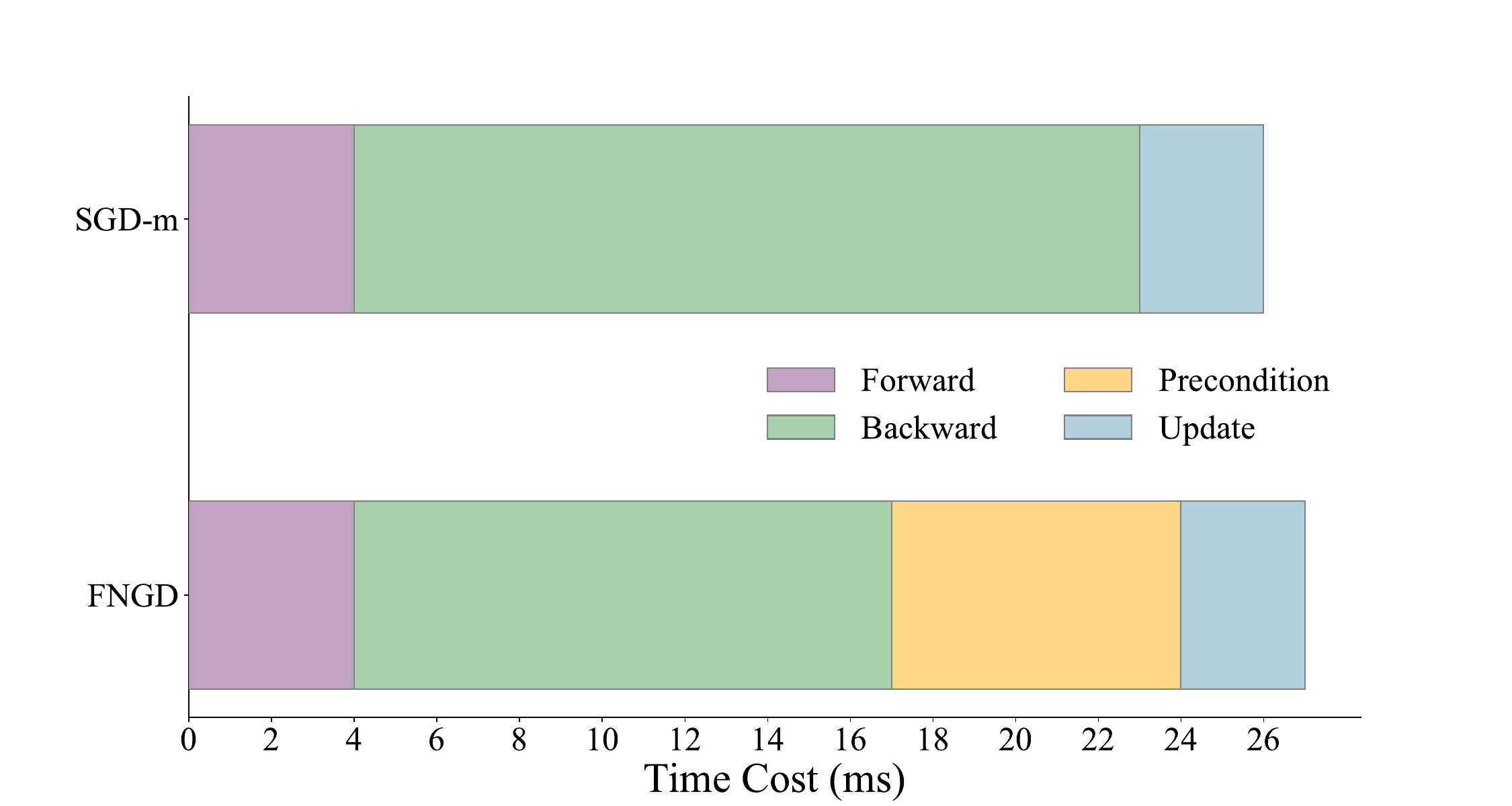}
    }
    
    \subfloat[Transformer]{
    \centering
    \includegraphics[scale=0.19]{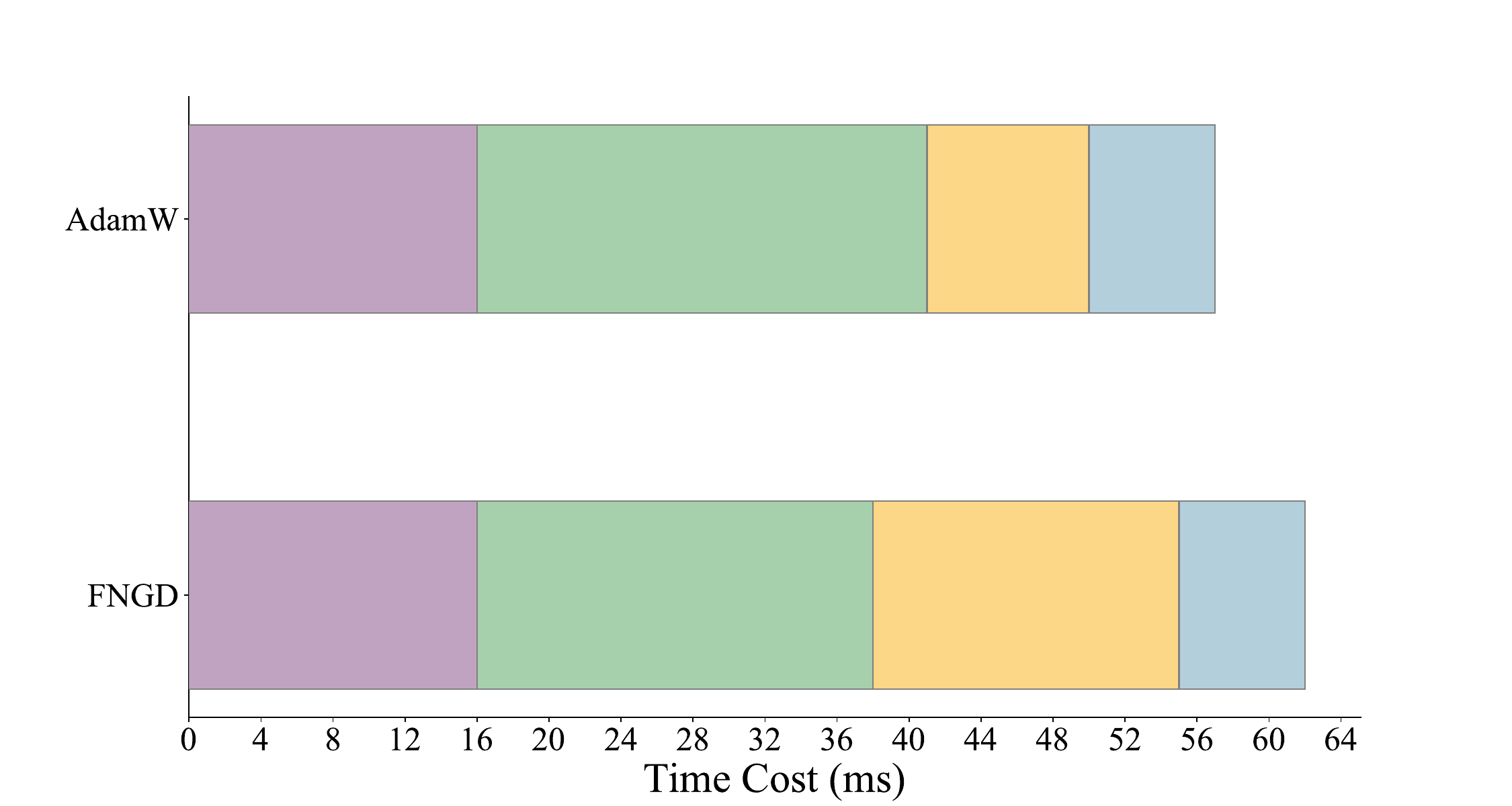}
    }
    
    \caption{Time Analysis of FNGD and first-order methods on two types of model structures.}
    \label{time_analysis}
\end{figure}

\subsection{Ablation Study}
We perform an ablation study to validate the necessity of the three key components in FNGD: coefficient-sharing, efficient precondition for acceleration, and damping determination. Without employing damping determination, we assigned a constant damping value of 0.3 to all layers in this experiment. Without applying efficient precondition, we start by performing the standard backward pass, leveraging hooks to derive the per-sample gradients, and subsequently compute the matrix-vector multiplication. The results are demonstrated in \cref{tab:ablation}.

It is demonstrated that any of the three key components is crucial for FNGD. Specifically, coefficient-sharing can reduce the training time by half with a small increase in accuracy. By employing Autograd on module output and performing efficient weighed-sum, a slight speedup can be achieved compared with using hooks. Without the damping determination strategy described in \cref{set_damping}, there will be an accuracy degeneration, particularly in the case of ResNet-18.
\begin{table}[htbp!]
    \caption{Ablation study on FNGD. We train ResNet-110 on Cifar-10, and ResNet-18 on CIFAR-100. The first row represents test accuracy, and the second row represents the relative per-epoch training time.}
    \centering
    \resizebox{0.45\textwidth}{!}{
    \begin{tabular}{c|c|c|c|c}
    \hline
    Network & FNGD& w/o sharing & w/o acceleration & w/o damping \\
    \hline
         \multirow{2}{*}{ResNet-32}& 93.10\%& 92.75\%& -& 92.63\% \\
    \cline{2-5}
          & 1.0$\times$ & 2.0$\times$ &1.3$\times$ &  -\\
    \hline
        \multirow{2}{*}{ResNet-18} & 77.37\%& 77.06\%& -& 75.51\%  \\
     \cline{2-5}
            &1.0$\times$ & 2.1$\times$ & 1.2$\times$& - \\
    \hline
    \end{tabular}}
    
    \label{tab:ablation}
\end{table}
\subsection{Hyper-parameter}

We explore the impact of different hyperparameter settings on the generalization performance of FNGD in comparison with KFAC. Our particular focus lies on the damping value. With the damping determination strategy, the damping value tuned in FNGD is referred to as the $\alpha$ in \cref{prop}. We perform the experiment by training ResNet-18 on CIFAR-100. For fairness, the other hyperparameters are fixed for both FNGD and KFAC.

We represent the accuracy of FNGD and KFAC with different damping values in \cref{tab:damping}. Specifically, by varying the damping value from 0.01 to 0.001, the mean accuracy of FNGD reaches 77.22\%, while KFAC exhibits a lower mean value of 77.00\%. Moreover, the highest accuracy 77.56\% of FNGD is much higher than that of KFAC, 77.22\%. In general, FNGD is less likely to require tedious fine-tuning to achieve high accuracy.
\begin{table}[htbp!]
    \caption{Hyper-parameter study on damping value by training ResNet-18 on CIFAR-100.}
    \centering
    \resizebox{0.5\textwidth}{!}{
    \begin{tabular}{c|c|c|c|c|c}
    \hline
        Method & 0.01 & 0.008 & 0.005 & 0.003 & 0.001 \\
    \hline
         FNGD& 77.24\% & 77.09\% & 77.56\% & 77.34\% & 76.86\%\\
    \hline
         KFAC& 77.22\% & 76.96\%& 77.00\% & 76.71\% & 77.10\%\\
    \hline
    \end{tabular}}
    
    \label{tab:damping}
\end{table}

\section{Conclusion}
We presented a fast natural gradient descent (FNGD) method, which is computationally efficient for deep learning. We first proposed to reformulate the gradient preconditioning formula in the NGD as a weighted sum of per-sample gradients using the SMW formula. Furthermore, these weighted coefficients are shared across epochs without affecting empirical performance. As a result, the inverse operator involved in computing coefficients only needs to be performed during the first epoch, and the computational complexity of FNGD approaches that of first-order methods. Extensive experiments on training DNNs are conducted to demonstrate that our method can outperform widely used second-order methods in terms of per-epoch training time while achieving competitive convergence and generalization performance.
\ifCLASSOPTIONcaptionsoff
  \newpage
\fi



%


{
\small
\bibliographystyle{ieee_fullname}
\bibliography{ref}

\begin{thebibliography}{10}\itemsep=-1pt

\bibitem{amari1998natural}
Shun-Ichi Amari.
\newblock Natural gradient works efficiently in learning.
\newblock {\em Neural computation}, 10(2):251--276, 1998.

\bibitem{bahamou2022mini}
Achraf Bahamou, Donald Goldfarb, and Yi Ren.
\newblock A mini-block fisher method for deep neural networks.
\newblock {\em arXiv preprint arXiv:2202.04124}, 2022.

\bibitem{battiti1992first}
Roberto Battiti.
\newblock First-and second-order methods for learning: between steepest descent
  and newton's method.
\newblock {\em Neural computation}, 4(2):141--166, 1992.

\bibitem{botev2017practical}
Aleksandar Botev, Hippolyt Ritter, and David Barber.
\newblock Practical gauss-newton optimisation for deep learning.
\newblock In {\em International Conference on Machine Learning}, pages
  557--565. PMLR, 2017.

\bibitem{duchi2011adaptive}
John Duchi, Elad Hazan, and Yoram Singer.
\newblock Adaptive subgradient methods for online learning and stochastic
  optimization.
\newblock {\em Journal of machine learning research}, 12(7):2121--2159, 2011.

\bibitem{george2018fast}
Thomas George, C{\'e}sar Laurent, Xavier Bouthillier, Nicolas Ballas, and
  Pascal Vincent.
\newblock Fast approximate natural gradient descent in a kronecker factored
  eigenbasis.
\newblock {\em Advances in Neural Information Processing Systems}, 31, 2018.

\bibitem{goodfellow2016deep}
Ian Goodfellow, Yoshua Bengio, and Aaron Courville.
\newblock {\em Deep learning}.
\newblock MIT press, 2016.

\bibitem{grosse2016kronecker}
Roger Grosse and James Martens.
\newblock A kronecker-factored approximate fisher matrix for convolution
  layers.
\newblock In {\em International Conference on Machine Learning}, pages
  573--582. PMLR, 2016.

\bibitem{gupta2018shampoo}
Vineet Gupta, Tomer Koren, and Yoram Singer.
\newblock Shampoo: Preconditioned stochastic tensor optimization.
\newblock In {\em International Conference on Machine Learning}, pages
  1842--1850. PMLR, 2018.

\bibitem{hager1989updating}
William~W Hager.
\newblock Updating the inverse of a matrix.
\newblock {\em SIAM review}, 31(2):221--239, 1989.

\bibitem{he2016deep}
Kaiming He, Xiangyu Zhang, Shaoqing Ren, and Jian Sun.
\newblock Deep residual learning for image recognition.
\newblock In {\em Proceedings of the IEEE conference on computer vision and
  pattern recognition}, pages 770--778, 2016.

\bibitem{hinton2012neural}
Geoffrey Hinton, Nitish Srivastava, and Kevin Swersky.
\newblock Neural networks for machine learning lecture 6a overview of
  mini-batch gradient descent.
\newblock {\em Cited on}, 14(8):2, 2012.

\bibitem{kingma2014adam}
Diederik~P Kingma and Jimmy Ba.
\newblock Adam: A method for stochastic optimization.
\newblock {\em arXiv preprint arXiv:1412.6980}, 2014.

\bibitem{lin2017focal}
Tsung-Yi Lin, Priya Goyal, Ross Girshick, Kaiming He, and Piotr Doll{\'a}r.
\newblock Focal loss for dense object detection.
\newblock In {\em Proceedings of the IEEE international conference on computer
  vision}, pages 2980--2988, 2017.

\bibitem{loshchilov2017decoupled}
Ilya Loshchilov and Frank Hutter.
\newblock Decoupled weight decay regularization.
\newblock {\em arXiv preprint arXiv:1711.05101}, 2017.

\bibitem{ly2017tutorial}
Alexander Ly, Maarten Marsman, Josine Verhagen, Raoul~PPP Grasman, and Eric-Jan
  Wagenmakers.
\newblock A tutorial on fisher information.
\newblock {\em Journal of Mathematical Psychology}, 80:40--55, 2017.

\bibitem{martens2015optimizing}
James Martens and Roger Grosse.
\newblock Optimizing neural networks with kronecker-factored approximate
  curvature.
\newblock In {\em International conference on machine learning}, pages
  2408--2417. PMLR, 2015.

\bibitem{more2006levenberg}
Jorge~J Mor{\'e}.
\newblock The levenberg-marquardt algorithm: implementation and theory.
\newblock In {\em Numerical analysis: proceedings of the biennial Conference
  held at Dundee, June 28--July 1, 1977}, pages 105--116. Springer, 2006.

\bibitem{mu2022hylo}
Baorun Mu, Saeed Soori, Bugra Can, Mert G{\"u}rb{\"u}zbalaban, and Maryam~Mehri
  Dehnavi.
\newblock Hylo: a hybrid low-rank natural gradient descent method.
\newblock In {\em SC22: International Conference for High Performance
  Computing, Networking, Storage and Analysis}, pages 1--16. IEEE, 2022.

\bibitem{nocedal1999numerical}
Jorge Nocedal and Stephen~J Wright.
\newblock {\em Numerical optimization}.
\newblock Springer, 1999.

\bibitem{pauloski2020convolutional}
J~Gregory Pauloski, Zhao Zhang, Lei Huang, Weijia Xu, and Ian~T Foster.
\newblock Convolutional neural network training with distributed k-fac.
\newblock In {\em SC20: International Conference for High Performance
  Computing, Networking, Storage and Analysis}, pages 1--12. IEEE, 2020.

\bibitem{ren2019efficient}
Yi Ren and Donald Goldfarb.
\newblock Efficient subsampled gauss-newton and natural gradient methods for
  training neural networks.
\newblock {\em arXiv preprint arXiv:1906.02353}, 2019.

\bibitem{robbins1951stochastic}
Herbert Robbins and Sutton Monro.
\newblock A stochastic approximation method.
\newblock {\em The annals of mathematical statistics}, pages 400--407, 1951.

\bibitem{tang2021skfac}
Zedong Tang, Fenlong Jiang, Maoguo Gong, Hao Li, Yue Wu, Fan Yu, Zidong Wang,
  and Min Wang.
\newblock Skfac: Training neural networks with faster kronecker-factored
  approximate curvature.
\newblock In {\em Proceedings of the IEEE/CVF Conference on Computer Vision and
  Pattern Recognition}, pages 13479--13487, 2021.

\bibitem{vaswani2017attention}
Ashish Vaswani, Noam Shazeer, Niki Parmar, Jakob Uszkoreit, Llion Jones,
  Aidan~N Gomez, {\L}ukasz Kaiser, and Illia Polosukhin.
\newblock Attention is all you need.
\newblock {\em Advances in neural information processing systems}, 30, 2017.

\bibitem{yao2021adahessian}
Zhewei Yao, Amir Gholami, Sheng Shen, Mustafa Mustafa, Kurt Keutzer, and
  Michael Mahoney.
\newblock Adahessian: An adaptive second order optimizer for machine learning.
\newblock In {\em proceedings of the AAAI conference on artificial
  intelligence}, volume~35, pages 10665--10673, 2021.

\bibitem{opacus}
Ashkan Yousefpour, Igor Shilov, Alexandre Sablayrolles, Davide Testuggine,
  Karthik Prasad, Mani Malek, John Nguyen, Sayan Ghosh, Akash Bharadwaj,
  Jessica Zhao, Graham Cormode, and Ilya Mironov.
\newblock Opacus: {U}ser-friendly differential privacy library in {PyTorch}.
\newblock {\em arXiv preprint arXiv:2109.12298}, 2021.

\bibitem{zadrozny2003cost}
Bianca Zadrozny, John Langford, and Naoki Abe.
\newblock Cost-sensitive learning by cost-proportionate example weighting.
\newblock In {\em Third IEEE international conference on data mining}, pages
  435--442. IEEE, 2003.

\bibitem{zhang2019fast}
Guodong Zhang, James Martens, and Roger~B Grosse.
\newblock Fast convergence of natural gradient descent for over-parameterized
  neural networks.
\newblock {\em Advances in Neural Information Processing Systems}, 32, 2019.

\bibitem{zhang2023eva}
Lin Zhang, Shaohuai Shi, and Bo Li.
\newblock Eva: A general vectorized approximation framework for second-order
  optimization.
\newblock {\em arXiv preprint arXiv:2308.02123}, 2023.

\end{thebibliography}
}
\clearpage
\onecolumn
\renewcommand\thesection{\Alph{section}}
\setcounter{table}{0}   
\setcounter{figure}{0}
\setcounter{section}{0}
\setcounter{equation}{0}
\setcounter{theorem}{0}

\begin{appendices}
\section{Proof of Theorem \ref{th1}}
We follow the framework established in \cite{zhang2019fast} to provide a theoretical convergence analysis. For simplicity, we consider the single-output case with squared error loss and the full-batch case.

First of all, with $0 \leq s \leq 1$, we define a linear interpolation between $\mb{w}_k$ and $\mb{w}_{k+1}$:
\begin{align*}
    \mb{w}_{k}(s) &= s\mb{w}_{k+1} + (1-s)\mb{w}_k \\
    &= \mb{w}_k - s\frac{\eta}{M}(\mb{F}_k+\lambda \mb{I})^{-1}\mb{J}_k(\mb{v}(\mb{w}_{k})-\mb{y})
\end{align*}
Then, we have:
\begin{equation}
\begin{aligned}
    \mb{v}(\mb{w}_{k+1})-\mb{v}(\mb{w}_{k})
    &= \mb{v}(\mb{w}_{k}-\frac{\eta}{M}(\mb{F}_k+\lambda \mb{I})^{-1}\mb{J}_k(\mb{v}(\mb{w}_{k})-\mb{y})) - \mb{v}(\mb{w}_{k})\\
    &= -\int_{s=0}^{1}\langle\frac{\partial \mb{v}(\mb{w}_k(s))}{\partial \mb{w}^{\text{T}}}, \frac{\eta}{M}(\mb{F}_k+\lambda \mb{I})^{-1}\mb{J}_k(\mb{v}(\mb{w}_{k})-\mb{y})\rangle ds \\
    &= -\int_{s=0}^{1}\langle\frac{\partial \mb{v}(\mb{w}_k)}{\partial \mb{w}^{\text{T}}}, \frac{\eta}{M}(\mb{F}_k+\lambda \mb{I})^{-1}\mb{J}_k(\mb{v}(\mb{w}_{k})-\mb{y})\rangle ds \\
    &+ \int_{s=0}^{1}\langle\frac{\partial \mb{v}(\mb{w}_k)}{\partial \mb{w}^{\text{T}}} - \frac{\partial \mb{v}(\mb{w}_k(s))}{\partial \mb{w}^{\text{T}}}, \frac{\eta}{M}(\mb{F}_k+\lambda \mb{I})^{-1}\mb{J}_k(\mb{v}(\mb{w}_{k})-\mb{y})\rangle ds.
\end{aligned}
\end{equation}
We denote the first integral term as $\mb{\Delta}_1$, and the second one as $\mb{\Delta}_2$. Consequently, we have:
\begin{align*}
    \mb{\Delta}_1 = \frac{\eta}{M}\mb{J}_k^{\text{T}}(\mb{F}_k+\lambda \mb{I})^{-1}\mb{J}_k(\mb{y}-\mb{v}(\mb{w}_{k})).
\end{align*}
Now, we bound the norm of $\mb{\Delta}_2$:
\begin{equation}
\begin{aligned}
    \|\mb{\Delta}_2\|_2 &\leq \frac{\eta}{M} \| \int_{s=0}^{1} \mb{J}(\mb{w}_{k}) - \mb{J}(\mb{w}_{k}(s)) ds\|_2 \|(\mb{F}_k+\lambda \mb{I})^{-1}\mb{J}_k(\mb{v}(\mb{w}_{k})-\mb{y}) \|_2 \\
    &\overset{(a)}{\leq} \frac{\eta}{M}\frac{\lambda_{\min}}{\sqrt{\lambda_{\max}}} \|(\mb{F}_k+\lambda \mb{I})^{-1}\mb{J}_k(\mb{v}(\mb{w}_{k})-\mb{y}) \|_2 \\
    &\overset{(b)}{\leq} \frac{\eta}{M}\frac{\lambda_{\min}}{\sqrt{\lambda_{\max}}}\| \frac{1}{\lambda}\Tilde{\mb{J}}\|_2 \| \mb I - (\lambda M \mb I + \mb G)^{-1}\mb G\|_2 \|\mb{K} (\mb{v}(\mb{w}_{k})-\mb{y}) \|_2 \\
    &\overset{(c)}{ = }\frac{\eta\lambda_{\min}}{M\sqrt{\lambda_{\max}}} \sqrt{\frac{\lambda M }{\lambda M + \lambda_{\min}}}\| \frac{1}{\lambda}\Tilde{\mb{J}}\|_2 \|\mb{K} (\mb{v}(\mb{w}_{k})-\mb{y}) \|_2 \\
    & \overset{(d)}{=} \frac{\eta \lambda_{\min}\sqrt{L}}{\sqrt{\lambda^2 M^2 + \lambda_{\min}\lambda M}} \|\mb{v}(\mb{w}_{k})-\mb{y} \|_2.
\end{aligned}
\end{equation}
where in (a) we use Assumption \ref{as2}, which implies
\begin{align*}
    \| \int_{s=0}^{1} \mb{J}(\mb{w}_{k}) - \mb{J}(\mb{w}_{k}(s)) ds\|_2 
    &\leq \|  \mb{J}(\mb{w}_{k}) -  \mb{J}(\mb{w}_{0})\|_2 +  \|  \mb{J}(\mb{w}_{k+1}) -  \mb{J}(\mb{w}_{0})\|_2 \\
    &\leq \frac{\lambda_{\min}}{\sqrt{\lambda_{\max}}}
\end{align*}
The inequality (b) follows the fact that
\begin{align*}
    (\mb{F}_k+\lambda \mb{I})^{-1}\mb{J}_k(\mb{v}(\mb{w}_{k})-\mb{y})
    &= (\frac{1}{M}\Tilde{\mb{J}}\Tilde{\mb{J}}^{\text{T}}+\lambda \mb{I})^{-1} \Tilde{\mb{J}}\mb{K} (\mb{v}(\mb{w}_{k})-\mb{y}) \\
    & =\frac{1}{\lambda}(\mb{I} - \frac{1}{M}\Tilde{\mb{J}}(\lambda \mb I + \frac{1}{M} \Tilde{\mb{J}}^{\text{T}}\Tilde{\mb{J}})^{-1}\Tilde{\mb{J}}^{\text{T}})\Tilde{\mb{J}}\mb{K} (\mb{v}(\mb{w}_{k})-\mb{y}) \\
    &=\frac{1}{\lambda}\Tilde{\mb{J}}(\mb I - (\lambda M \mb I + \mb G)^{-1}\mb G)\mb{K} (\mb{v}(\mb{w}_{k})-\mb{y}) 
\end{align*}
In the equality (c), we use the fact that 
\begin{align*}
    \| \mb I - (\lambda M \mb I + \mb G)^{-1}\mb G\|_2 
    &= \sigma_{\max}(\mb I - (\lambda M \mb I + \mb G)^{-1}\mb G)\\
    &=\sqrt{\frac{\lambda M}{\lambda M + \lambda_{\min}}}
\end{align*}
Here, we derive the maximal singular value of $\mb I - (\lambda M \mb I + \mb G)^{-1}\mb G$ using Lemma \ref{lem1}. We have
\begin{align*}
    \lambda_{max}(\mb I - (\lambda M \mb I + \mb G)^{-1}\mb G) &= \max_{\mu} 1-\frac{\mu}{\lambda M + \mu} \\
    &=  \frac{\lambda M}{\lambda M + \lambda_{\min}}
\end{align*}
\begin{lemma}
\label{lem1}
 The set of eigenvectors for $\mb G$ is equivalent to the set of eigenvectors for $(\mb G + \lambda M \mb I)^{-1} \mb G$. For eigenvalues, we have $\lambda((\mb G + \lambda M \mb I)^{-1} \mb G)= \frac{\mu}{\lambda M + \mu} $, where $\mu$ denotes the eigenvalue of $\mb G$. \\
\end{lemma}
\begin{proof}
On the one hand, we can prove that the eigenvectors of $\mb{G}$ are the eigenvectors of $(\mb G + \lambda M \mb I)^{-1} \mb G$. With $(\mu, \mb p)$ denotes the (eigenvalue, eigenvector) pair of $\mb G$, we have:
\begin{align*}
    &\mb G \mb p = \mu \mb p \\
    &\Rightarrow(\mb G + \lambda M \mb I) \mb p = (\mu + \lambda M)\mb p \\
    &\Rightarrow\frac{\mu}{\mu + \lambda M} \mb p = (\mb G + \lambda M \mb I)^{-1} \mb G \mb p
\end{align*}
On the other hand, we will prove that the eigenvectors of $(\mb G + \lambda M \mb I)^{-1} \mb G$ are also the eigenvectors of $\mb G$. With $(\alpha, \mb q)$ denotes the (eigenvalue, eigenvector) pair of $(\mb G + \lambda M \mb I)^{-1} \mb G$, we have:\\
\begin{align*}
    &(\mb G + \lambda M \mb I)^{-1} \mb G \mb q = \alpha \mb q \\
    &\Rightarrow \mb G \mb q = \alpha (\lambda M \mb I + \mb G) \mb q \\
    &\Rightarrow (1-\alpha)\mb G \mb q = \alpha\lambda M \mb q \\
    &\Rightarrow \mb G \mb q = \frac{\alpha\lambda M}{1-\alpha} \mb q
\end{align*}
In conclusion, eigenvectors sets of the two matrices are the same.\\
\end{proof}
In the equality (d), we use the facts that $\| \Tilde{\mb{J}}\|_2 = \sqrt{\lambda_{\max}( \mb{G})} = \sqrt{\lambda_{max}}$ and $\|\mb{K} (\mb{v}(\mb{w}_{k})-\mb{y}) \|_2 = \sqrt{L} \|\mb{v}(\mb{w}_{k})-\mb{y}\|_2$.\\
Finally, we have:
\begin{align*}
    ||\mb{v}_{k+1} - \mb{y}||_2^2 &= || \mb{v}_{k}-\mb{y}+\mb{v}_{k+1} - \mb{v}_{k}||_2^2 \\
    &= || \mb{v}_{k}-\mb{y}||_2^2 -2(\mb y - \mb{v}_k)^{\text{T}}(\mb{v}_{k+1}-\mb{v}_{k})+||\mb{v}_{k+1} - \mb{v}_{k}||_2^2 \\
    &\leq || \mb{v}_{k}-\mb{y}||_2^2 - \frac{2\eta}{M}\underbrace{(\mb y - \mb{v}_k)^{\text{T}} \mb{J}_k^{\text{T}}(\mb{F}_k+\lambda \mb{I})^{-1}\mb{J}_k(\mb{y}-\mb{v}_{k}))}_{\textcircled{1}} + \frac{2\eta \lambda_{\min}\sqrt{L}}{\sqrt{\lambda^2 M^2 + \lambda_{\min}\lambda M}} \|\mb{v}_{k}-\mb{y} \|_2^2 + \underbrace{\|\mb{v}_{k+1} - \mb{v}_{k}\|_2^2}_{\textcircled{2}}\\
    &\leq (1-\frac{2\lambda_{\min}\eta L}{\lambda M + \lambda_{\min}} +\frac{2\eta \lambda_{\min}\sqrt{L}}{\sqrt{\lambda^2 M^2 + \lambda_{\min}\lambda M}} + \eta^2(  L\sqrt{\frac{\lambda_{\max}}{\lambda^2 M^2 + \lambda_{\max}\lambda M}} + \frac{\lambda_{\min}\sqrt{L}}{\sqrt{\lambda^2 M^2 + \lambda_{\min}\lambda M}})^2)\|\mb{v}_{k}-\mb{y} \|_2^2 \\
    &= (1-\eta) \|\mb{v}_{k}-\mb{y} \|_2^2 \\ 
    &+ \eta(\eta ( L\sqrt{\frac{\lambda_{\max}}{\lambda^2 M^2 + \lambda M\lambda_{\max}}} + \frac{\lambda_{\min}\sqrt{L}}{\sqrt{\lambda^2 M^2 + \lambda_{\min}\lambda M}})^2 - (\frac{2\lambda_{\min}L}{\lambda M + \lambda_{\min}} - \frac{2 \lambda_{\min}\sqrt{L}}{\sqrt{\lambda^2 M^2 + \lambda_{\min}\lambda M}} -1 ))\|\mb{v}_{k}-\mb{y} \|_2^2
\end{align*}
The part $\textcircled{1}$ is lower bounded as follows:
\begin{align*}
    (\mb y - \mb{v}_k)^{\text{T}} \mb{J}_k^{\text{T}}(\mb{F}_k+\lambda \mb{I})^{-1}\mb{J}_k(\mb{y}-\mb{v}_{k})) 
    & \geq \lambda_{\min} (\Tilde{\mb{J}}_k^{\text{T}}(\mb{F}_k+\lambda \mb{I})^{-1}\Tilde{\mb{J}}_k) \|\mb K (\mb{v}_k - \mb y)\|_2^2\\
    &= L\lambda_{\min} (\Tilde{\mb{J}}_k^{\text{T}}(\mb{F}_k+\lambda \mb{I})^{-1}\Tilde{\mb{J}}_k) \|\mb{v}_k - \mb y\|_2^2 \\
    &= \frac{\lambda_{\min} M L}{\lambda M + \lambda_{\min}}  \|\mb{v}_k - \mb y\|_2^2
\end{align*}
For the third line, we have:
    \begin{align*}
    \Tilde{\mb{J}}_k^{\text{T}}(\mb{F}_k+\lambda \mb{I})^{-1}\Tilde{\mb{J}}_k
    &= \frac{1}{\lambda}\Tilde{\mb{J}}^{\text{T}}(\mb{I} - \frac{1}{M}\Tilde{\mb{J}}(\lambda \mb I + \frac{1}{M} \Tilde{\mb{J}}^{\text{T}}\Tilde{\mb{J}})^{-1}\Tilde{\mb{J}}^{\text{T}})\Tilde{\mb{J}} \\
    &= \frac{1}{\lambda}\mb{G}(\mb{I}-(\lambda M \mb I + \mb G)^{-1}\mb G),
\end{align*}
and the minimal eigenvalue of $\mb{G}(\mb{I}-(\lambda M \mb I + \mb G)^{-1}\mb G)$ is obtained by Lemma \ref{lemma2}.
\begin{lemma}\label{lemma2}
   The set of eigenvectors for $\mb G$ is equivalent to the set of eigenvectors for $\mb{G}(\mb{I}-(\lambda M \mb I + \mb G)^{-1}\mb G)$. For eigenvalues, we have $\lambda(\mb{G}(\mb{I}-(\lambda M \mb I + \mb G)^{-1}\mb G))= \frac{\mu\lambda M}{\lambda M + \mu} $, where $\mu$ denotes the eigenvalue of $\mb G$. 
\end{lemma}
\begin{proof}
On the one hand, we can prove that the eigenvectors of $\mb{G}$ are the eigenvectors of $\mb{G}(\mb{I}-(\lambda M \mb I + \mb G)^{-1}\mb G)$.\\
Let $(\mu, \mb p)$ denotes the (eigenvalue, eigenvector) pair of $\mb G$. With the help of \cref{lem1}, we have:
\begin{align*}
&\mb G \mb p = \mu \mb p \\
&\Rightarrow(\mb G + \lambda M \mb I)^{-1} \mb G \mb p = 
   \frac{\mu}{\mu + \lambda M} \mb p \\
&\Rightarrow \mb I - (\mb G + \lambda M \mb I)^{-1} \mb G \mb p = \frac{\lambda M}{\mu + \lambda M} \mb p \\
&\Rightarrow\mb G( \mb I - (\mb G + \lambda M \mb I)^{-1} \mb G) \mb p = \frac{\lambda M}{\mu + \lambda M} \mb G\mb p\\
&\Rightarrow\mb G( \mb I - (\mb G + \lambda M \mb I)^{-1} \mb G) \mb p = \frac{\mu\lambda M}{\mu + \lambda M}\mb p
\end{align*}
On the other hand, we will prove that the eigenvectors of $\mb{G}(\mb{I}-(\lambda M \mb I + \mb G)^{-1}\mb G)$ are also the eigenvectors of $\mb G$. With $(\alpha, \mb q)$ denotes the (eigenvalue, eigenvector) pair of $\mb{G}(\mb{I}-(\lambda M \mb I + \mb G)^{-1}\mb G)$, we have:\\
\begin{align*}
    &\mb G( \mb I - (\mb G + \lambda M \mb I)^{-1} \mb G) \mb q = \alpha \mb q \\
    & \Rightarrow (\mb I - \mb G(\mb G + \lambda M \mb I)^{-1})\mb G \mb q= \alpha \mb q \\
    &  \Rightarrow (\mb (\mb G + \lambda M \mb I)(\mb G + \lambda M \mb I)^{-1} - \mb G(\mb G + \lambda M \mb I)^{-1})\mb G \mb q= \alpha \mb q \\
    &\Rightarrow \lambda M(\mb G + \lambda M \mb I)^{-1}\mb G \mb q= \alpha \mb q \\
    &\Rightarrow \mb G \mb q= \frac{\alpha}{\lambda M} (\mb G + \lambda M \mb I)\mb q \\
    &\Rightarrow \mb G \mb q= \frac{\alpha {\lambda} M}{\lambda M - \alpha}\mb q
\end{align*}
Hence, we conclude the proof.   
\end{proof}
The part $\textcircled{2}$ is upper bounded using inequality (1) and (2), as follows:
\begin{align*}
    \|\mb{v}_{k+1} - \mb{v}_{k}\|_2
    & \leq \frac{\eta}{M}\|\mb{J}_k^{\text{T}}(\mb{F}_k+\lambda \mb{I})^{-1}\mb{J}_k(\mb{y}-\mb{v}_{k}))\|_2
    + \|\Delta_2\|_2 \\
    & \leq \frac{\eta L}{M} \|\Tilde{\mb{J}}_k^{\text{T}}(\mb{F}_k+\lambda \mb{I})^{-1}\Tilde{\mb{J}}_k)\|_2 \|\mb{v}_{k}-\mb{y} \|_2 + \frac{\eta \lambda_{\min}\sqrt{L}}{\sqrt{\lambda^2 M^2 + \lambda_{\min}\lambda M}}\|\mb{v}_{k}-\mb{y} \|_2 \\
    & =(  \frac{\eta L}{\lambda M} \sqrt{\frac{\lambda_{\max} \lambda M}{\lambda M + \lambda_{\max}}} + \frac{\eta \lambda_{\min}\sqrt{L}}{\sqrt{\lambda^2 M^2 + \lambda_{\min}\lambda M}})\|\mb{v}_{k}-\mb{y} \|_2
\end{align*}
In the equality above, the maximal singular value of $\Tilde{\mb{J}}_k^{\text{T}}(\mb{F}_k+\lambda \mb{I})^{-1}\Tilde{\mb{J}}_k)$ can be derived by Lemma \ref{lemma2}.

Let us consider the function $\operatorname{f}(\lambda M):=\frac{2\lambda_{\min}L}{\lambda M + \lambda_{\min}} - \frac{2\lambda_{\min}\sqrt{L}}{\sqrt{\lambda^2 M^2 + \lambda_{\min}\lambda M}} -1$. We have that\\
\begin{align*}
    \operatorname{f}(\lambda_{\min})=L-\sqrt{2L}-1\geq 0 \quad\text{for}\quad L\geq 4.
\end{align*}
For such choice of damping value, and for a small enough learning rate, \ie\\
\begin{align}
    \eta \leq \frac{L-\sqrt{2L}-1}{(L\sqrt{\frac{\lambda_{\max}}{\lambda_{\min}^2+\lambda_{\min}\lambda_{\max}}}+\frac{\sqrt{2L}}{2})^2} = \Tilde{\eta},
\end{align}
we can get that 
\begin{align*}
    \|\mb{v}_{k+1}-\mb{y} \|_2^2 \leq (1-\eta)\|\mb{v}_k-\mb{y} \|_2^2,
\end{align*}
which concludes the proof of \cref{th1}.

So far, we have assumed $\mb{w}_{k+1}$ fall within a certain radius around the initialization. We now justify this assumption.
\begin{lemma}
    If Assumptions \ref{as1} and \ref{as2} hold, when the optimization process converges to $\mb{w}_{k}$, we have 
\begin{align*}
    \| \mb{w}_{k+1} - \mb{w}_0 \|_2 \leq \frac{\sqrt{\lambda_{\max}L}}{\lambda_{\min}}\|\mb{v}(\mb w_0)-\mb y\|_2.  
\end{align*}
\end{lemma}
\begin{proof}
\begin{align*}
    \| \mb{w}_{k+1} - \mb{w}_0 \|_2 &\leq \frac{\eta}{M} \sum_{i=0}^{k} \|(\mb F_i + \lambda \mb I)^{-1} \mb J_i (\mb{v}(\mb w_i)-\mb y) \|_2\\
    &\leq \frac{\eta}{\lambda M}\sqrt{\frac{\lambda_{\max}\lambda M L}{\lambda M + \lambda_{\min}}}\sum_{i=0}^{k}  \|\mb{v}(\mb w_i)-\mb y\|_2 \\
    &= \frac{\eta\sqrt{\lambda_{\max}L}}{\sqrt{2}\lambda_{\min}}\sum_{i=0}^{k}  \|\mb{v}(\mb w_i)-\mb y\|_2 \\
    &\leq \frac{\eta\sqrt{\lambda_{\max}L}}{\sqrt{2}\lambda_{\min}}\sum_{i=0}^{k} (1-\eta)^{\frac{i}{2}} \|\mb{v}(\mb w_0)-\mb y\|_2 \\
    &< \frac{\sqrt{\lambda_{\max}L}}{\lambda_{\min}}\|\mb{v}(\mb w_0)-\mb y\|_2.
\end{align*}
\end{proof}
\end{appendices}

\end{document}